\documentclass{article}
\usepackage{arxiv}

\usepackage[round]{natbib}
\usepackage{algorithm}
\usepackage[noend]{algpseudocode}
\usepackage{amssymb}
\usepackage{amsmath}
\usepackage{amsthm}
\usepackage{mathabx}

\usepackage{threeparttable}
\usepackage{arydshln}
\usepackage{graphicx}
\usepackage{epstopdf}
\usepackage{tcolorbox}
\usepackage{caption}
\usepackage{subcaption}
\usepackage{arydshln}
\usepackage{multirow}
\usepackage{enumitem}
\usepackage{booktabs}
\usepackage{mathtools, cuted}
\usepackage{tablefootnote}
\usepackage{fancyhdr} 

\usepackage{hyperref}       
\usepackage{url}            
\usepackage[T1]{fontenc}
\usepackage[utf8]{inputenc}

\newtheorem{theorem}{Theorem}

\newtheorem{corollary}{Corollary}[theorem]
\newtheorem{proposition}[theorem]{Proposition}
\newtheorem{remark}{Remark}[theorem]
\newtheorem{definition}{Definition}[section]

\usepackage{amsmath,amsfonts,bm}

\def\eqref#1{equation~\ref{#1}}

\def\1{\bm{1}}

\def\v0{{\bm{0}}}

\def\vc{{\bm{c}}}

\def\vu{{\bm{u}}}
\def\vv{{\bm{v}}}
\def\vw{{\bm{w}}}
\def\vx{{\bm{x}}}
\def\vy{{\bm{y}}}
\def\vz{{\bm{z}}}
\def\vomega{{\bm{\omega}}}

\def\mA{{\bm{A}}}
\def\mB{{\bm{B}}}
\def\mC{{\bm{C}}}

\def\mI{{\bm{I}}}

\def\mK{{\bm{K}}}

\def\mW{{\bm{W}}}
\def\mX{{\bm{X}}}

\DeclareMathAlphabet{\mathsfit}{\encodingdefault}{\sfdefault}{m}{sl}
\SetMathAlphabet{\mathsfit}{bold}{\encodingdefault}{\sfdefault}{bx}{n}

\def\gN{{\mathcal{N}}}
\def\gO{{\mathcal{O}}}
\def\gP{{\mathcal{P}}}

\def\gS{{\mathcal{S}}}

\def\gX{{\mathcal{X}}}
\def\gY{{\mathcal{Y}}}

\def\sC{{\mathbb{C}}}

\def\sN{{\mathbb{N}}}

\def\sR{{\mathbb{R}}}

\newcommand{\norm}[1]{\left\| #1 \right\|}

\def \a2 {\mA_2}
\def \qa2 {\mA_{2q}}
\def \b2 {\mB_2}
\def \qb2 {\mB_{2q}}
\def \c2 {\mC_2}
\def \qc2 {\mC_{2q}}

\def \omn {\omega_n}

\def \smqs#1 { \sum_{k=0}^{p-1} \sum_{r=\lceil -k/n\rceil }^{\lfloor (p-1-k)/n \rfloor} t_k^{2q} t_{n r+k}^#1 e_{s,k} }

\def \tr#1 {\text{tr}\left(#1\right)}

\def \p#1 {{\left(#1\right)}}
\def \norms#1 {\left\|#1\right\|^2}
\def \bs#1 {{\left[#1\right]}}
\def \br#1 {{\left\{#1\right\}}}
\def \inp#1 {{\langle#1\rangle}}

\def \suma#1 { \sum_{k=0}^{n-1} \p{ \sum_{\nu=0}^{l-1}  t_{k+n\nu}^#1 } }
\def \sumc#1 {  \sum_{k=0}^{n-1} \p{\sum_{\nu =l}^{\tau -1} t_{k+n\nu}^#1 } }
\def \las#1 { \suma#1 \p{ \sum_{j=0}^{n-1}  \omn^{(s-k)j}} }
\def \lcs#1 { \sumc#1 \p{ \sum_{j=0}^{n-1}  \omn^{(s-k) j}} }

\def \sinsuma#1 { \sum_{\nu=0}^{l-1}  t_{k+n\nu}^#1  }
\def \sinsumc#1 { \sum_{\nu =l}^{\tau  -1} t_{k+n\nu}^#1 } 

\def \a {\alpha}

\def \vcts#1 {\vc_{t, #1}}

\def \gail#1 {{\color{blue}(#1)}}

\def \inner#1 {{\langle#1\rangle}}

\def \abs#1{{\left|#1\right|}}
\def \fstar{f^{\star}}
\def \fsstar{f_s^{\star}}
\def \fsharp{f^{\sharp}}
\def \vcstar{\vc^{\star}}
\def \vcstars{\vc_s^{\star}}
\def \vcsharp{\vc^{\sharp}}
\def \sstar{\gS^{\star}}
\def \ssharp {\gS^{\sharp}}
\def \vcp{\vc^{\sharp}|_{\ssharp}}
\def \fsp{\fsharp_{\ssharp}}

\def \tvert#1{{\left\vvvert#1\right\vvvert}}
\def \teta{{\tilde{\eta}}}

\title{SHRIMP: Sparser Random Feature Models via \\ Iterative Magnitude Pruning }

\author{
Yuege Xie \thanks{Equal contribution. Correspondence to: Yuege Xie.} \\
\texttt{yuege@oden.utexas.edu} \\
University of Texas at Austin\\
\And Bobby Shi$^*$ \\
\texttt{bhshi@utexas.edu} \\
University of Texas at Austin\\
\And  Hayden Schaeffer\\
\texttt{hschaeff@andrew.cmu.edu}\\
Carnegie Mellon University \\
\And Rachel Ward\\
\texttt{rward@math.utexas.edu}\\
University of Texas at Austin\\
}

\begin{document}
\maketitle
\begin{abstract}
Sparse shrunk additive models and sparse random feature models have been developed separately as methods to learn low-order functions, where there are few interactions between variables, but neither offers computational efficiency. On the other hand, $\ell_2$-based shrunk additive models are efficient but do not offer feature selection as the resulting coefficient vectors are dense. Inspired by the success of the iterative magnitude pruning technique in finding lottery tickets of neural networks, we propose a new method---\textbf{S}parser \textbf{R}andom Feature Models via \textbf{IMP} (ShRIMP)\footnote{Code and examples are available at \url{https://github.com/rhshi/sparse-rf}.}---to efficiently fit high-dimensional data with inherent low-dimensional structure in the form of sparse variable dependencies. Our method can be viewed as a combined process to construct and find sparse lottery tickets for two-layer dense networks. We explain the observed benefit of SHRIMP through a refined analysis on the generalization error for thresholded Basis Pursuit and resulting bounds on eigenvalues.

From function approximation experiments on both synthetic data and real-world benchmark datasets, we show that SHRIMP obtains better than or competitive test accuracy compared to state-of-art sparse feature and additive methods such as SRFE-S, SSAM, and SALSA. Meanwhile, SHRIMP performs feature selection with low computational complexity and is robust to the pruning rate, indicating a robustness in the structure of the obtained subnetworks. We gain insight into the lottery ticket hypothesis through SHRIMP by noting a correspondence between our model and weight/neuron subnetworks.
\end{abstract}

\section{Introduction}\label{sec:intro}
Kernel regression is an established choice for learning a target function from data with solid theoretical foundation \citep{hearst1998support, zhang2005learning}.  Kernel methods are general-purpose methods for estimating a function by fitting measurements to a representative function in a Reproducing Kernel Hilbert Space \citep{campbell2002kernel}.  The power of the method is derived from the Representer Theorem which connects finite measurements and continuous function space; however, kernel ridge regression does not take into account additional structure in the underlying target function, which can be limiting when additional structure is known to be present.  In many physical settings \citep{harris2019additive}, functions may arise naturally as sums of functions, each with a limited variable interaction.  Such low-order structure \citep{on_decompositions} may also be used to reduce the complexity of the system being modeled \citep{potts2019approximation, potts2021interpretable}. The Multiple Kernel Learning line of literature \citep{gonen2011multiple, bach2008consistency, xu2010simple} and more recent methods such as shrunk additive models \citep{kandasamy2016additive, liu2020sparse}  have been developed to exploit such low-order structure. However, these methods are computationally inefficient due to the naive cost of kernel ridge regression and minimal $\ell_1$-norm optimization \citep{liu2020sparse}, as well as repeated computation of the kernel at test and prediction periods.

In an independent line of work, the well-studied random features model as introduced in  \citet{rahimi_random_2008} allows for an approximation to the kernel function of interest without the computational cost of constructing a full kernel matrix, with bounds given in \cite{rahimi2008weighted, 4797607, avron2017random, pmlr-v9-cortes10a}.  However, neither the generic random features model nor the shrunk additive model offers the possibility of simple model compression or feature selection.  Kernel methods have been investigated in the context of feature selection \citep{pmlr-v5-kumar09a} and in sparse additive models \citep{10.1214/09-AOS781, yin2012group, ravikumar2009sparse}, and recent work by \citet{hashemi2021generalization} introduces the sparse random features model with coefficient vector recovered using basis pursuit. However, it is a priori unclear what coefficient sparsity means, especially in the $\ell_1$ sense, for random feature models.

Inspired by the success of the iterative magnitude pruning (IMP) technique for finding sparse subnetworks of neural networks with comparable performance \citep{frankle_lottery_2019, zhou2019deconstructing}, we propose a new $\ell_2$-based method---Sparser Random Feature Models via Iterative Magnitude Pruning (SHRIMP)---to efficiently fit high-dimensional data with inherent low-order structure. This method can be viewed as a combined process to construct and find lottery tickets of two-layer dense networks: it randomly initializes a fixed first layer with only low-order interactions by using sparse random weights instead of dense weights (Stage I) and applies neuron pruning to find a sparser sub-network by IMP (Stage II). From experiments on both synthetic data and real-world benchmark datasets, we show that SHRIMP is better or competitive against both random sparse feature models and shrunk additive models. We offer a refined analysis of the thresholded $\ell_1$-based sparse random feature model based on \cite{hashemi2021generalization}, and through our experiments and further discussion offer insight into the success of our method.

We connect our work to the larger active literature on random features learning and the theory of neural networks.  Neural networks in certain regimes have been found to be kernel machines \citep{jacot_neural_2020, chizat2018lazy}, and thus linear regression has been revisited as a model for neural network behavior \citep{Liang_2020, hastie_surprises_2019}.  In particular, there has been a concentrated interest in the generalization and double descent behaviors of random features regression \citep{montanari_generalization_2019, jacot2020implicit, dascoli_double_2020}; these works study ReLU features, while the closest model to our setting is that of \citet{liao2020random}.  \emph{As random features regression has become essential for studying neural networks in the kernel regime, we hope that random features pruning can be used to study neural network pruning}; much of the current theory on neural network pruning has been existential in nature \citep{orseau2020logarithmic, malach_proving_2020, pensia2020optimal}.  Moreover, the lottery ticket hypothesis has one of the greatest implications for model compression, helping alleviateproblems related to massive neural networks, such as unequal access to computing resources among researchers, and the environmental impact of deep learning.

Our main contributions are as follows:
\begin{enumerate}

    \item We propose a two-stage algorithm, SHRIMP, to learn sparse random feature models with low-order interactions efficiently via an iterative magnitude pruning technique. It is surprising that $\ell_2$-based SHRIMP finds sparser models than basis pursuit, which is a staple among algorithms used for sparse recovery and feature selection.
   
   \item Experiments on both synthetic and real-world datasets verify that with proper choice of the order $q$ in the low-order function model, SHRIMP obtains better or matching test performance compared to existing sparse feature or shrunk additive models while at the same time being scalable to high-dimensional settings. Beyond the effectiveness and efficiency, we show that SHRIMP is robust to parameters such as pruning rate and exhibits surprising support recovery ability with odd/even separation.
   
   \item We offer a refined analysis of the main theorem from \citet{hashemi2021generalization} that allows us to connect our experimental findings with initial theoretical results from the compressed sensing literature; one perspective of our method is that it situates itself between $\ell_2$ and $\ell_0$-based methods. Our analysis on the evolution of the spectrum of the Gram matrix during the SHRIMP algorithm iterations offers additional insight into the benefit of SHRIMP over random pruning.

   \item By connecting our SHRIMP model to the process of finding winning lottery tickets for a two-layer fully connected neural network initialized with a random sparse subnetwork and with neurons pruned by IMP, we shed light on the successful performance of IMP as a mechanism for finding lottery tickets. 
  
\end{enumerate}


\begin{figure}[ht]
    \centering
    \includegraphics[width=.85\linewidth]{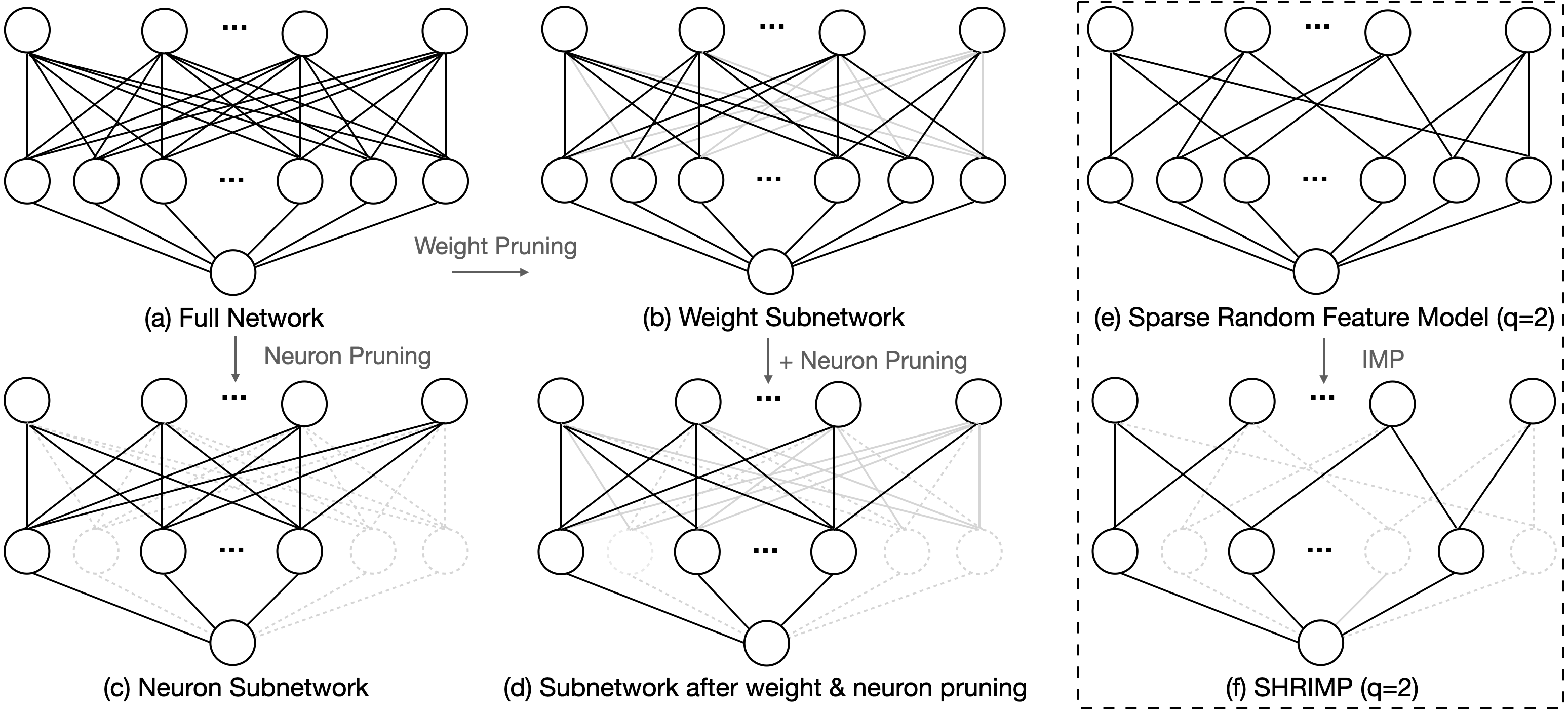}
    \caption{Illustration of the relationship between SHRIMP and different types of pruning.}
    \label{fig:pruning}
\end{figure}

\subsection{Related Work}
\textbf{Sparse Random Feature Models and Shrunk Additive Models.} In high-dimensional settings arising from modeling physical systems, the underlying governing function is well-approximated as being a sum of \textit{low-order} functions; that is, the function can be written as a sum of component functions, such that only $q$ variables out of a total of $d$ variables are active in each component, with $q\ll d$ \citep{low_order, on_decompositions}.  Recent methods such as SALSA \citep{kandasamy2016additive} and SSAM \citep{liu2020sparse} are kernel-based methods that directly exploit such low-order structure.  Separately, the random Fourier features model of  \citet{rahimi_random_2008} is a popular choice in approximating a kernel function when there are many data samples and the construction of the kernel matrix is computationally expensive.  Additive kernels were considered in \cite{6136519}, while coefficient sparsity was investigated in \citep{NIPS2014_459a4ddc, ozcelikkale2020sparse}.  The first work (to our knowledge) to combine these approaches--exploiting low-order structure and random features together --was the work of \citet{hashemi2021generalization}, which uses an $\ell_1$-based approach with \textit{sparse} random features in order to learn a low-order function. A detailed comparison of SHRIMP with sparse random feature models \citep{hashemi2021generalization, elesedy2020lottery} and shrunk additive models \citep{kandasamy2016additive, liu2020sparse} is listed in Table \ref{tab:comp}.

\textbf{Lottery Ticket Hypothesis and Iterative Magnitude Pruning.} \citet{frankle_lottery_2019} propose the lottery ticket hypothesis and a corresponding iterative magnitude pruning procedure for compressing neural networks by pruning weights based on their magnitude and retraining pruned subnetworks from the same initial weights at each pruning iteration to find a sparse subnetwork (winning ticket) with comparable test accuracy to the original overparameterized dense neural network. Follow-up work \citep{zhou2019deconstructing, ramanujan_whats_2020, malach_proving_2020} prove that for a sufficiently overparameterized neural network, a good initial sub-network with random weights achieves competitive accuracy compared to the original network. \citet{elesedy2020lottery} initiated a theoretical analysis on a simplified form of IMP where one prunes a single weight per iteration and the solution is obtained by gradient flow (i.e., min $\ell_2$-norm estimator) in linear models.  However, this is not a practical sparse estimation method in linear models due to the inefficiency of one-weight-per-iteration pruning and computation of gradient descent with $t\rightarrow\infty$. Instead, our work applies proportional IMP to neuron pruning and utilizes the implicit regularization of the pseudo-inverse to reduce the computational cost significantly. The work \cite{zhang2021lottery} analyzes the geometric structure of the model throughout pruning; we corroborate these ideas by providing results on the eigenvalues of the Gram matrix throughout pruning.  However, our work does not focus solely on the lottery ticket hypothesis, but should serve as a stand-alone method for low-order function approximation.

\begin{table}[ht]
    \centering
    \caption{Comparison of sparse random feature models: SRFE-S \citep{hashemi2021generalization}, shrunk additive models such as SALSA \citep{kandasamy2016additive}, SSAM \citep{liu2020sparse}, and IMP in Linear Regression \citep{elesedy2020lottery}. Here, $d$ denotes the data dimension, $q$ denotes the interaction order of features, and $T_p$ indexes the pruning iteration. The feature number with $N_s$ corresponds to the number of non-zero components in the feature vector by $\ell_1$-regularization; $N_{best}$ denotes the number of features remaining after IMP.}
    \scalebox{1}{
    \begin{tabular}{lccccc}
    \toprule
     Property & SHRIMP & SALSA & SSAM & SRFE-S & IMP in LR\\
     \midrule
      Sample sparsity  &  $\surd$ & $\times$ & $\surd$ & $\surd$ & $\times$  \\
      Low-order Interaction & $\surd$ & $\surd$ & $\surd$ & $\surd$ & $\times$\\
      Feature sparsity  & $\surd$ & $\times$ & $\surd$ & $\surd$ & $\surd$ \\
      Computational Efficiency & $\surd$ & $\surd$ & $\times$ & $\times$ & $\surd$ \\
      Regularization & implicit $\ell_2$ & $\ell_2$-norm & $\ell_1$-norm & $\ell_1$-norm & implicit $\ell_2$ \\
      Feature model &  random feature & kernel & kernel & random feature & linear \\
      $\#$ Features* & $ \binom{d}{q} \cdot n \rightarrow  N_{best}$ & $\binom{d}{q}$ & $\binom{d}{q} \rightarrow N_s$ & $\binom{d}{q} \cdot n \rightarrow  N_s$ & $d \rightarrow d-T_p$  \\
  \bottomrule
    \end{tabular} }
    \label{tab:comp}
\end{table}

\subsection{Notation}
Throughout the paper, $\mA^\dagger$ denotes the Moore–Penrose inverse of a matrix $\mA$, and hence $\vc = \mA^\dagger \vy$ is the minimal $\ell_2$-norm estimator in the overparameterized regime, and the least-square estimator in the underparameterized regime. We denote the number of data points by $m$, the dimension of data by $d$, and the number of features by $N$. Measurement noise $e_k$ is defined as $y_k = f(\vx_k) + e_k$ with either $|e_k| \leq E =2\nu$ or $e_k$ i.i.d. drawn from $\gN(0, \nu^2)$, $\forall k \in [m]$.

\section{Preliminaries}
Low order functions arise naturally in the physical world and are used as a form of reduced-complexity model for such systems \citep{potts2019approximation, potts2021interpretable}. Let us first recall the definition for an order-$q$ function, as well as definitions for bounded $\rho$-norm functions and $q$-sparse feature weights from \cite{hashemi2021generalization}.

\begin{definition}[\textbf{Order-$q$ Function}]\label{def:order_q_func} For any $d, q, K \in\mathbb{N_+}$ with $q\leq d$, a function $f:\sC^d\to\sC$ is an order-$q$ function of at most $K$ terms if there exist functions $g_1,\dots, g_K:\sC^q\to \sC$ such that 
    \begin{equation}
        f(x_1,\dots, x_d)=\frac{1}{K}\sum_{j=1}^{K}g_j(x_{j_1},\dots, x_{j_q})=\frac{1}{K}\sum_{j=1}^{K}g_j(\vx|_{\mathcal{S}_j}),
    \end{equation}
    where $\mathcal{S}_j\subseteq [d]$ is an index subset of $[d]$ and $\vx|_{\mathcal{S}_j}$ is the restriction of $\vx$ onto the indices.
\end{definition}

In general, such a decomposition is not unique.  However, the set of order-$q$ functions forms a vector space, as the sum of two order-$q$ functions is itself an order-$q$ function, and the space is closed under scalar multiplication.  Additionally, if we let $\norm{\cdot}$ be a function norm, then we can define
\[\norm{f} = \inf \sqrt{\frac{1}{K}  (\norm{g_1}^2+\dots+\norm{g_K}^2)} ,\]
where the infimum is taken over all possible order-$q$ decompositions of $f$.  If each $g_j$ lies in a Reproducing Kernel Hilbert Space (RKHS), then $f$ lies in the direct sum of the component Reproducing Kernel Hilbert Spaces with RKHS norm defined as above \citep{aronszajn1950theory}.

\begin{definition}[\textbf{Bounded $\rho$-norm Function}]\label{def:function_class1}
Fix a probability density function $\rho: \mathbb{R}^d \rightarrow \mathbb{R}$ and a function $\phi:  \mathbb{R}^d \times \mathbb{R}^d \rightarrow \mathbb{C}$.  A function $f: \mathbb{R}^d \rightarrow \mathbb{C}$ has finite $\rho$-norm with respect to $\phi( \vx; \vw )$ if it belongs to the class
\begin{align}
\mathcal{F}({\phi,\rho}) := \Bigg\{f(\vx) = \int_{\vw \in \mathbb{R}^d} \alpha(\vw) \phi(  \vx; \vw ) d\vw : \|f\|_\rho := \sup_\vw \left|\frac{ \alpha(\vw)}{\rho(\vw)}\right|<\infty\Bigg\} .
\end{align}
\end{definition}

\begin{definition}[\textbf{$q$-sparse Feature Weights}]\label{def:sparse_weights}
    Let $d, q, n\in\mathbb{N_+}$ with $q\leq d$, and let $\rho:\mathbb{R}^q\to\mathbb{R}$ be a probability distribution.  A collection of $N=n\binom{d}{q}$ weight vectors $\vomega_1,\dots, \vomega_N$ is called a set of $q$-sparse feature weights if it is  generated as follows: for each index subset $\mathcal{S}_j\subseteq [d], |\mathcal{S}_j|=q$ draw $n$ i.i.d. random vectors $\mathbf{z}_1,\dots, \mathbf{z}_n\sim \rho$, and construct $q$-sparse features $\{\vomega_{j_k}\}_{k=1}^n$ by setting $\mathrm{supp}(\vomega_{j_k})=\mathcal{S}_j$ and $\vomega_{j_k}|_{\mathcal{S}_j}=\mathbf{z}_k, k \in [n]$.
\end{definition}

Let $\mX \in \sR^{m\times d}$ be a data matrix consisting of $m$ $d$-dimensional samples, and let $\mW \in \sR^{N\times d}$ be the matrix of $N$ $d$-dimensional feature weights constructed according to Def. \ref{def:sparse_weights}.  Construct the random feature matrix $\mA = \phi(\mX\mW^\ast )$ so that $\mA \mA^\ast $ approximates a kernel matrix of interest. For example, let $\phi (\cdot) = [\cos(\cdot), \sin(\cdot)]$, $q=d$, and $\rho$ be the normal distribution, $\mathbf{E}_W[\mA \mA^\ast ]$ is the kernel matrix of the Gaussian kernel \citep{rahimi_random_2008}. In the more general case of $q \leq d$, with $\rho$ as a normal distribution, $\mathbf{E}_W[\mA \mA^\ast ]$ is the kernel matrix corresponding to a \textit{direct sum} of Gaussian kernels, each defined over $\mathbb{R}^q\times \mathbb{R}^q$. See the appendix for more precise approximation bounds.

\section{Sparser Random Feature Models via Iterative Magnitude Pruning} \label{sec:shrimp}
In this section, we first present the proposed SHRIMP algorithm; we then illustrate its connection to the lottery ticket hypothesis and network pruning. To address the challenge of targeting low-order additive structure efficiently with feature selection, we propose a two-step SHRIMP method (Algorithm \ref{imp}) to find a sparse low-order random feature subnetwork of a fully connected neural network: first, we initialize a sparse random feature model by constructing low-order random feature weights as a subnetwork of the dense network, according to Def. \ref{def:sparse_weights}; second, SHRIMP finds a sparse winning lottery ticket $\vc_{t^{\star}}$ by forming a set of sparse min $\ell_2$-norm estimators $\{\vc_t \}_{t=1}^T$ via IMP and selecting the best model via a validation dataset. At test time, with the best model $\{\vc_{t^{\star}}, \gP_{t^{\star}}\}$ chosen, we transform the test data via  $\mA^{test} =  \phi(\mX^{test}_{\gP_{t^{\star}}} \mW_{\gP_t^{\star}}^\ast )$ and predict via $\vy^{test} = \mA^{test} \vc_{t^{\star}}$. For synthetic data, the validation and test data are randomly drawn from the same distribution as the training data; for real-world data, we randomly split the training data into training and validation sets.

\begin{algorithm}[ht]
\caption{SHRIMP: Sparser Random Feature Models via Iterative Magnitude Pruning}
\label{imp}
    \begin{algorithmic}[1]
    \State \textbf{Input:} Training dataset $(\gX, \gY) \in \sC^{m \times d} \times \sC^m$, parametric basis function $\phi(\cdot~; \vw)$, feature sparsity level $q$, pruning rate $p \in (0,1)$, iterations of pruning $T \in \sN_+$.
    
    \State \textbf{Stage I. Constructing: } Draw $N = n \binom{d}{q}$ $q$-sparse features $\{\vw_j\}_{j=1}^N$ according to Def. \ref{def:sparse_weights}, form the matrix $\mW \in \sR^{N\times d}$, and construct a random feature matrix $\mA \in \sR^{m\times N}$ by
    $\mA= \phi(\mX \mW^\ast )$.
    
    \State \textbf{Stage II. Pruning:} Compute the min $\ell_2$-norm solution $\vc_0 = \mA^{\dagger} \vy$ and set $\gP_0 = \{1, \dots, N\}$.
    
    \For{ $t = 1, \ldots, T$ }
      \State Get the feature index set $\gP_t$ by pruning $p \times |\gP_{t-1}|$ features from $\vc_{t-1}$ with $\vc^s_t$ denoting the ascending sorted (by absolute value) array of $\vc_t$ by  $$\gP_t = \{i \in \gP_{t-1} :  |\vc_{t-1, i}| \geq \vc^s_{t-1, p\times|\gP_{t-1}|} \}.$$
      
      \State Update the min $\ell_2$-norm solution: $\vc_t = \mA_{\gP_t}^{\dagger} \vy$, where $\mA_{\gP_t}$ is the column submatrix of $\mA$ with index set $\gP_t$.
    \EndFor
    
    \State \textbf{Output:} the pruned minimal $\ell_2$ norm estimator and feature index pair set $\{(\vc_t, \gP_t) \}_{t=1}^T$.
    \end{algorithmic}
\end{algorithm}

\textbf{Discussion on Computational Complexity.}  At each step $t$, SHRIMP gets the min $\ell_2$-norm solution by solving a linear system $\vc_t = \mA_{\gP_t}^\dagger \vy$ rather than computing pseudo-inverse directly, so each step has computational cost \textit{at most  $\gO(N_tm^2+m^3)$} in the worst case (it is possible to leverage random sketching to solve $\mA_{\gP_{t+1}}$ based on the previous computation of $\mA_{\gP_t}$).  Additionally, $|\gP_t|= N_t = N(1-p)^t$ and $t \in [0, T]$ with $p \in (0,1)$ and $(1-p)^T N= 1$. Hence, $N_t$ is at most $\sum_{t=0}^{T} N(1-p)^t = N\frac{(1-1/N)}{p} = \frac{N-1}{p} \ll NT$. Meanwhile, the complexity of $\ell_1$ minimization is known to be at least polynomial in $N$; for example, the complexity of interior-point methods to obtain the min $\ell_1$-norm estimator is $\gO(N^6)$. SRFE is solved using the spgl1 package in Python/MATLAB~\citep{spgl1site, BergFriedlander:2008}. A detailed comparison of computational time required by  SHRIMP and SRFE with different scales of data is in Figure \ref{fig:time}.

\textbf{Connection to Neural Network Pruning.} Consider a two-layer fully connected neural network $f(\vx)$ with activation function $\phi(\cdot)$ (see also Figure \ref{fig:pruning}(a)),
\begin{align}
    f(\vx) = \sum_{i=1}^N a_i \phi (\vw_i^\ast \vx).
\end{align}
Winning tickets are defined as sparse subnetworks that reach test accuracy comparable to the original network~\citep{frankle_lottery_2019}.
Finding winning lottery tickets is known to be computationally hard in the worst case~\citep{frankle_lottery_2019,malach_proving_2020,zhang2021efficient}, and this poses a challenge to understanding why certain pruning methods tend to work well in practice.   Pruning methods for neural networks  generally fall into two categories: weight pruning and neuron pruning \citep{malach_proving_2020}.  Weight pruning (Figure \ref{fig:pruning} (a) to (b)) involves a set of binary mask vectors $\vu_i\in\{0, 1\}^d$, equivalent to pruning the first layer weights, resulting in the network
\begin{equation}
    \tilde{f_w}(\vx) = \sum_{i=1}^N a_i \phi((\vw_i \odot \vu_i )^\ast \vx),
\end{equation}
while neuron pruning (Figure \ref{fig:pruning} (a) to (c)) involves a set of binary scalars $b_i\in \{0, 1\}$, equivalent to pruning entire neurons, resulting in the network 
\begin{equation}
    \tilde{f_n}(\vx) = \sum_{i=1}^N (b_i a_i) \phi (\vw_i^\ast \vx).
\end{equation}

A sparse subnetwork is the result of applying both types of pruning,  as shown in Figure \ref{fig:pruning}(d). The SHRIMP method first fixes the weight subnetwork at a specified sparsity level determined by the choice of low-order parameter $q$ (Figure \ref{fig:pruning}(e)); it then adaptively ``prunes" the neurons by finding a sparse coefficient vector (Figure \ref{fig:pruning}(f)).  In other words, for each subset $\mathcal{S}_j\subseteq [d], |\mathcal{S}_j|=q$, we define $\vu_j\in\{0, 1\}^d, \mathrm{supp}(\vu_j)=\mathcal{S}_j$, and then adaptively prune $\vc$ so that the result is
\begin{equation}
     f^\star(\vx)=\frac{1}{K}\sum_{j=1}^K \sum_{\ell=1}^n (b_{j, \ell}c_{j,\ell}) \phi((\vw_{j, \ell} \odot \vu_j )^\ast \vx),
\end{equation}
where $K=\binom{d}{q}$.  Using the notation of Algorithm \ref{imp}, if $s$ is the sparsity level of $\vc_{t^\star}$, then we can compress the number of nonzero entries of $\mW_{\gP_t^\star}, \vc_{t^\star}$ to $(q+1)s$. This is in contrast to the standard random feature model with dense $\mW, \vc$, where there are $(d+1)N$ non-zeros. Note that SHRIMP performs best when the target function is a sum of low-order components, so that the selected model is actually a pruned sparse model. Experiments shown in Table \ref{tab:compare_func} corroborate this finding, with ``Avg size'' ($\#$features) much smaller than $N$. Explicitly capturing the low-order structure is an advantage of SHRIMP.

\section{Experiments}\label{sec:exp}
We now show generalization error results on synthetic and real-world datasets. We then illustrate several benefits of  SHRIMP compared to other approaches including computational efficiency, robustness to pruning rate, and sparse support recovery, as well as other benefits of iterative magnitude pruning.  
\subsection{Function Approximation}

To demonstrate the performance of SHRIMP on low-order functions, we first test different models on synthetic functions with $\vx=[x_1, x_2, \dots, x_d]^T$. See the appendix for additional experiments.
\begin{itemize}
    \item Simple additive functions: $f_1(\vx)=\sum_{i=1}^{d-1}x_i+\exp(-x_d)$ and $f_2(\vx)=\cos(x_1)+\sin(x_2)$;
    
    \item Functions with pairwise behavior (from \cite{liu2020sparse}): $f_3(\vx)=(2x_1-1)(2x_2-1)$ and $f_4(\vx)=(2x_1-1)(2x_2-1)+(2x_1-1)(2x_3-1)+(2x_2-1)(2x_3-1)$;
    
    \item Low-order non-smooth functions: $f_5(\vx)=\mathrm{sinc}(x_1)\mathrm{sinc}(x_3)^3+\mathrm{sinc}(x_2)$;
    
    \item Ishigami example used for uncertainty and sensitivity analysis \citep{151285}: $f_6(\vx)=\sin(x_1)+7\sin^2(x_2)+0.1x_3^4\sin(x_1)$;
    
    \item An order-2 function with many order-1 components: $f_7(\vx)=\cos(x_1)x_3+x_2^2x_4+\sum_{i=3}^d x_i$.
\end{itemize}

\begin{table*}[ht]
    \centering
      \caption{Comparing the test errors. Avg size denotes the average pruned model size of SHRIMP taken over three runs. $^\dagger s$ denotes the setting with different $(m, d, q)$ pairs: $s_l^{q_*} = (140, 10, q_*), s_h^{q_*} = (1400, 100, q_*)$, where $l$ denotes low dimension and $h$ denotes high dimension, and all methods use $q=q_*$ (the ground-truth order). The best MSE for each $f_i(\vx)$ over all models is in purple. From the comparison between $s_l^{q_*}$ and $s_h^{q_*}$, SHRIMP is the best overall, and is significantly more scalable to the high-dimensional setting.}
    \scalebox{1}
    {\begin{tabular}{clccccccc}
    \toprule
    Setting$^{\dagger}$ & Model  & $f_1(\vx)$
    & $f_2(\vx)$
    & $f_3(\vx)$
    & $f_4(\vx)$ 
    & $f_5(\vx)$
    & $f_6(\vx)$
    & $f_7(\vx)$
    \\ 
 
    \midrule

    & SRFE-S & 7.85e-04 & 1.98e-05 & 1.15e-01 & 1.27e-01  & 7.52e-03 & 1.11 & 7.49e-02\\
     & Min $\ell_2$ & 4.37e-20 & 5.45e-24 & 8.20e-02 & 5.94e-02 &  7.36e-03 & 7.18 & 2.98e-02\\
    $s_l^{q_*}$ & SALSA & 1.59e-12 & 1.26e-15 & 8.80e-02 & 6.14e-02 &  7.32e-03 & 6.99 & 2.72e-02\\
     & SHRIMP & \textcolor{purple}{1.37e-22} & \textcolor{purple}{7.90e-32} & \textcolor{purple}{4.98e-12} & \textcolor{purple}{2.54e-12}  & \textcolor{purple}{6.39e-04} & \textcolor{purple}{2.58e-02} & \textcolor{purple}{2.83e-05}\\
     \hdashline
    & Avg size* & 3100.33 & 29 & 147 & 171.67 & 39 & 80.33 & 187 \\
    
    \midrule
    & SRFE-S & 1.52e-03 & 8.71e-06 & 1.99 & 4.54&  1.16e-01 & 4.59 & 4.40e-01\\
    & Min $\ell_2$ & 1.68e-20 & 3.51e-24 & 2.01 & 4.75 &  1.16e-01 & 8.34 & 1.49e-01\\
    $s_h^{q_*}$& SALSA & 1.99e-11 & 2.54e-13 & 1.60 & 3.16 &  8.49e-02 & 7.35 & 1.34e-01\\
     & SHRIMP & \textcolor{purple}{1.61e-22} & \textcolor{purple}{1.11e-30} & \textcolor{purple}{1.26e-02} & \textcolor{purple}{5.11e-01} & \textcolor{purple}{1.50e-02} & \textcolor{purple}{2.68} & \textcolor{purple}{5.82e-02}\\
     \hdashline
    & Avg size & 3355 & 19 & 61.33 & 64  & 42.67 & 13 & 229\\
    \midrule
    & Order $q_*$ & 1 & 1 & 2 & 2 &3 &2 &2 \\
     \bottomrule
    \end{tabular}}
    \label{tab:compare_func}
\end{table*}
 
\textbf{Experimental Results.} Table \ref{tab:compare_func} reports the test mean-squared error (MSE) for the function approximations $\{f_i(\vx)\}_{i=1}^7$ and corresponding optimal $q$\footnote{See the appendix for more results on $q=d$.}. As shown in Table \ref{tab:compare_func}, SHRIMP consistently outperforms other function approximations, often by an order of magnitude. This is consistent across low- ($s_l^{q_*}$) and high-dimensional ($s_h^{q_*}$) settings. In particular, SHRIMP performs notably better than SALSA \citep{kandasamy2016additive}, which explicitly constructs the additive kernel matrix and performs kernel regression. One notable example is $f_6$, the Ishigami function; SHRIMP is the only method that attains a test MSE less than $1$ in the low-dimensional setting. In addition, SHRIMP succeeds in finding sparse models. The best-performing SHRIMP model often has orders of magnitude fewer features, performing significant model compression (from initial 20000 features to ``Avg size" in Table \ref{tab:compare_func}), and often moving from the overparameterized to the underparameterized setting via this compression. When more variables are involved, e.g., $f_1$, which is order-$1$ but involves all the variables, $\#$features retained also increases. On the other hand, for $f_2$, which is a simple sum involving two trigonometric terms, SHRIMP retains only $19$ features yet attains the best test error by a significant margin.

\subsection{Real-world Datasets}\label{sec:exp_real}
We test SHRIMP on eight real-world datasets from the UCI repository  (\url{http://archive.ics.uci.edu/ml}) and follow the experimental setup (\url{https://github.com/kirthevasank/salsa}) in \cite{kandasamy2016additive}. We compare the test errors with shrunk additive models, SALSA \citep{kandasamy2016additive} and SSAM \citep{liu2020sparse}, and the sparse model by Lasso. For the experiments with SHRIMP, we use $90\%$ original training dataset as training data and $10\%$ as validation data to select the best model from models trained with $q \in \{1, 2, 3, \dots, d\}$  and the pruning rate $p(\%) \in \{15, 25, 35\}$. (For practical consideration, we usually only use $q \in \{1, 2, \dots, \max\{10, d\}\}$.) From Table \ref{tab:real}, SHRIMP attains the best test errors on the Propulsion and Galaxy datasets and has comparable results on all other datasets while still being significantly more efficient to implement. 

\begin{table*}[ht]
    \centering
    \caption{Test MSE on eight real-world benchmark datasets. The results of SALSA, SSAM, and Lasso are taken from \cite{liu2020sparse} and \cite{kandasamy2016additive}. The best MSE for each dataset is in purple. }
    \scalebox{1}{
    \begin{tabular}{lllll}
    \toprule
   Dataset $(m, d)$ & SHRIMP $(q, n_b)$ & SALSA $(q)$  & SSAM & Lasso \\
    \midrule
    Propulsion $(200, 15)$ &  ${\color{purple}1.02 \times 10^{-6}}~(2, 84)$ & $8.81\times 10^{-3}~(8)$ & NA & $2.48 \times 10^{-2}$ \\

    Galaxy $(2000, 20)$ & $ {\color{purple} 5.41 \times 10^{-6}} ~(3, 575)$& $ 1.35\times 10^{-4} ~(4)$ & NA & $2.39 \times 10^{-2}$\\

    Airfoil $(750, 41)$ &  ${\color{purple}2.65 \times 10^{-1}} ~(2, 30)$ & $5.18 \times 10^{-1} ~(5)$ &  $4.87 \times 10^{-1}$ & $5.20 \times 10^{-1}$ \\

    CCPP $(2000, 59)$  &  ${\color{purple}6.55\times 10^{-2}} ~(2, 49)$ & $6.78\times 10^{-2} ~(2)$ & $6.94\times 10^{-2}$ & $7.40\times 10^{-2}$ \\

    Insulin $(256, 50)$ & $1.24 \times 10^{0} ~(1, 60)$ &  $1.02 \times 10^{0} ~(3)$ & ${\color{purple}1.01 \times 10^{0}} $ & $1.11 \times 10^{0}$ \\

    Telemonit $(1000, 19)$ & $6.00\times 10^{-2}~(4, 86)$ & ${\color{purple} 3.47\times 10^{-2}} ~(9)$ & $6.89\times 10^{-2}$ & $8.63\times 10^{-2}$ \\
    
    Housing $(256, 12)$ & $3.94\times 10^{-1}~(7, 15)$ & ${\color{purple}2.62\times 10^{-1}} ~(1)$  & $3.79\times 10^{-1}$ & $4.4\times 10^{-1}$ \\
    
    Skillcraft $(1700, 18)$ & $5.81\times 10^{-1} ~(8, 21)$ & $5.47\times 10^{-1} ~(1)$ & ${\color{purple}5.43\times 10^{-1}}$ & $6.65\times 10^{-1}$\\
    \bottomrule
    \end{tabular} }
    \label{tab:real}
\end{table*}

\subsection{Properties of SHRIMP}
\textbf{Computational Efficiency.} To compare the time efficiency of SHRIMP to SRFE-S ($\ell_1$-minimization), we approximate the function $f(\vx) = 3\cos(x_3) + 4\sin(x_4) + 2\sin(x_2)$  using varying values for the parameters $m, d$, and $N$. Fig. \ref{fig:time} shows that at equal parameter choices $m, d, N,$ SHRIMP is significantly faster than SRFE-S. In particular, as $m$ increases, the cost of SHRIMP over varying $N$ increases linearly, while the cost of basis pursuit increases exponentially. In addition to offering better generalization error, SHRIMP is significantly less computationally intensive.

\begin{figure}[ht]
    \centering
    \includegraphics[width=.56\linewidth]{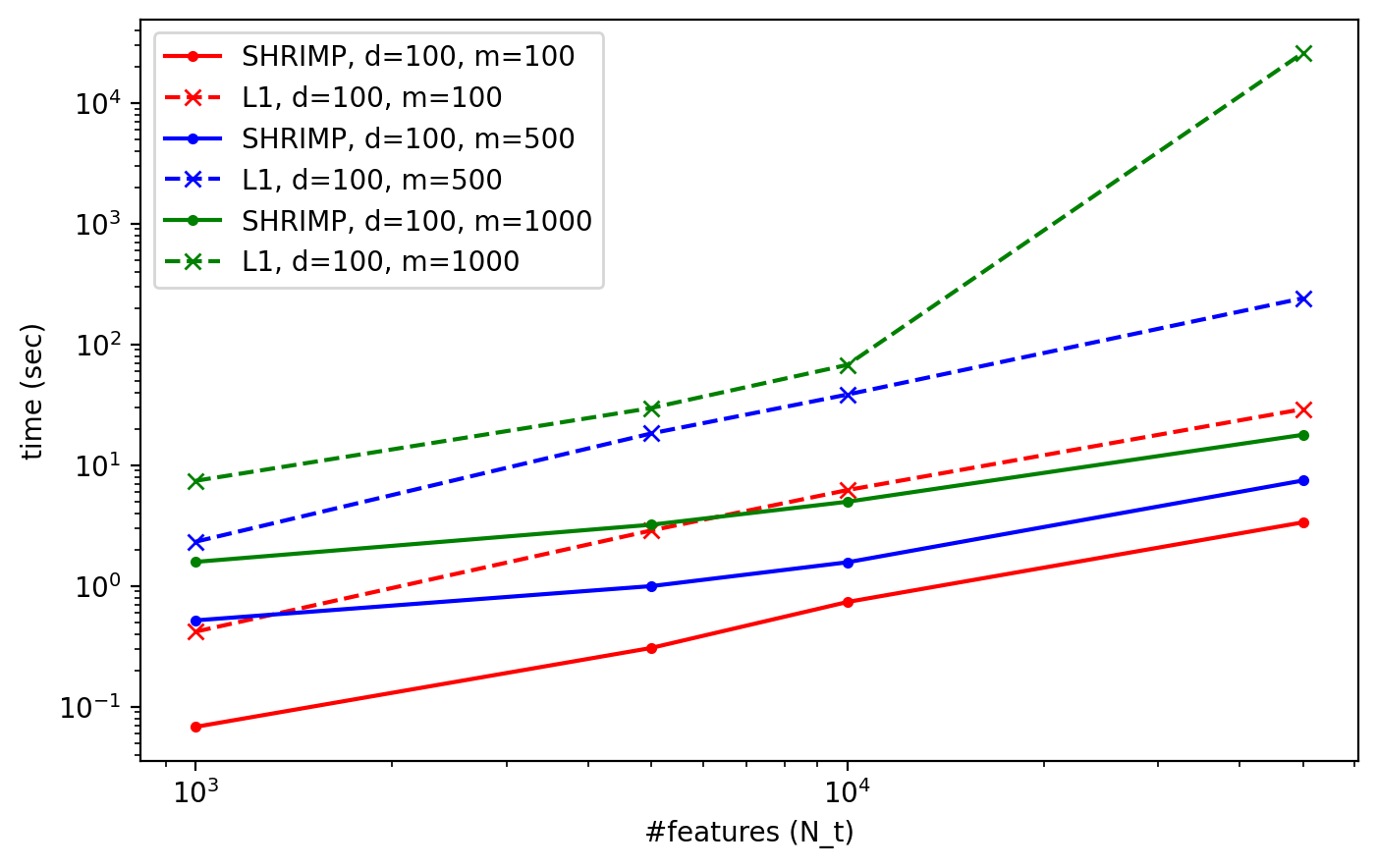}
    \caption{Time Comparison: SHRIMP vs. SRFE-S ($L_1$) as a function of $\#$features.}
    \label{fig:time}
\end{figure}
\textbf{Robustness to Pruning Rate.} We study the robustness of SHRIMP to the pruning rate on real-world datasets by using a range of pruning rates $p (\%) \in \{15, 20,$ $25, 30, 35, 40, 45, 50\}$ and comparing the best $q$ chosen by validation dataset and corresponding test MSE. Figure \ref{fig:robust} shows that the best $q$ is almost invariant over all pruning rates, and the corresponding test MSEs remain within a small range. Hence, SHRIMP exhibits robustness to the pruning rate, indicating a corresponding robustness in the structure of good subnetworks.

\begin{figure}[ht]
    \centering
    \includegraphics[width=.47\linewidth]{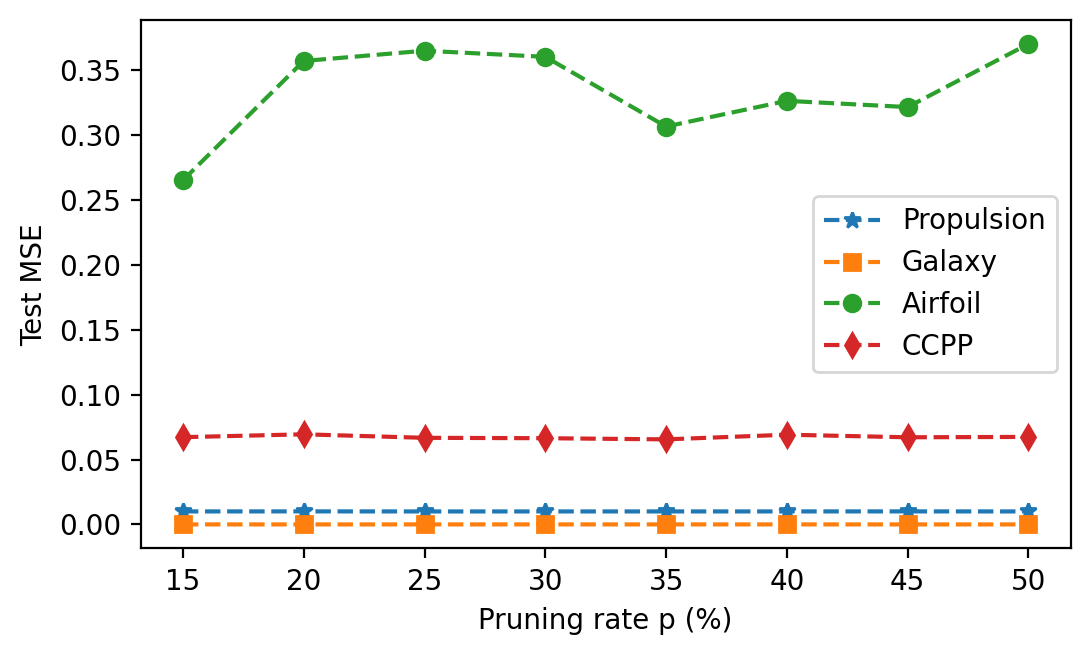}
    \includegraphics[width=.45\linewidth]{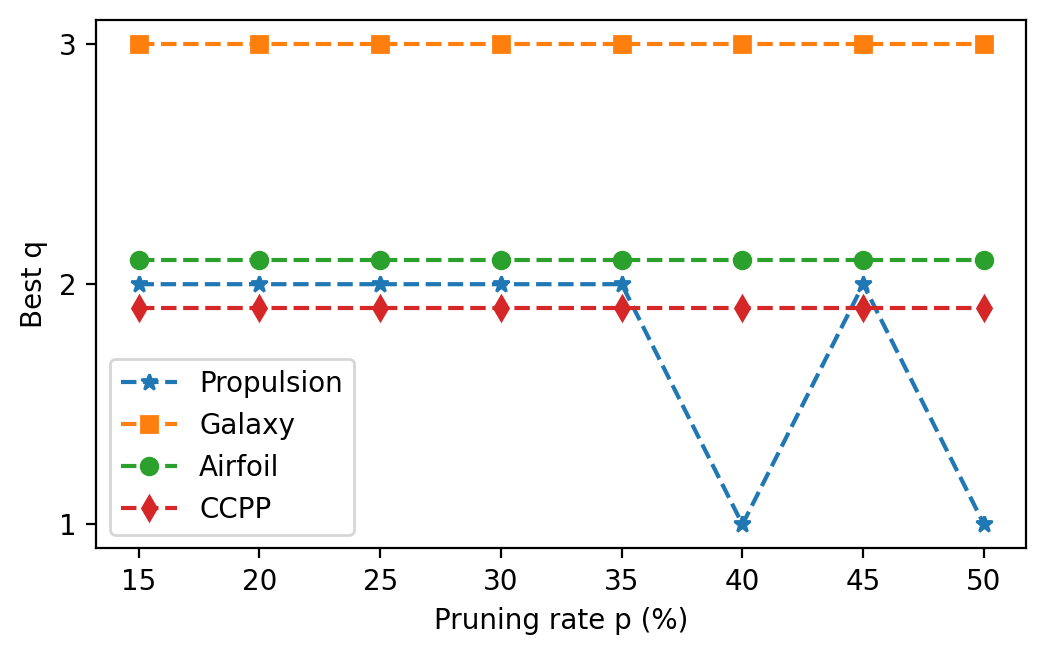}
    \caption{Robustness of pruning rate shown by test MSE and the corresponding best $q$ over four datasets. Note the lines of best $q$ of Airfoil and CCPP are at $2$ and shifted on purpose for illustration. Test MSE of propulsion is also shifted $0.01$ up for illustration. }
    \label{fig:robust}
\end{figure}

\textbf{Sparse Support Recovery.} We illustrate the power of SHRIMP as a method for sparse support recovery on a simple order-$1$ additive function $f_s(\vx) = 3\cos(x_3) + 4\sin(x_4) + 2\sin(x_2)$ with separate component functions on each coordinate. We sample $m=1000$ points uniformly from $[-1, 1]^5$ and apply SHRIMP with $N=10000$ ($N_0 = 20000$ features), $q=q_*=1$, and pruning rate $p=20\%$. 
Figure \ref{fig:support} shows that both the sparse support set and the even/odd property are recovered by SHRIMP.  At first, when $N_t=20000$, the min $\ell_2$-norm solution has many small weights distributed across false coordinates $\{x_3, x_5\}$. At the first key point $N_t=8192$, most unnecessary weights on $\sin$ have been pruned; after $N_t=879$, remaining weights are only on $\{x_2, x_3, x_4\}$ with correct even/odd partition. The best test MSE is at $N_t=38$, where the resulting vector is extremely sparse and matches the support set exactly. Along the pruning process, the subnetworks found by SHRIMP maintain a comparative or better test error with only a small fraction of weights. In other words, winning tickets found by SHRIMP exhibit the ability to recover sparse low-order interactions in random feature models.

\begin{figure}[ht]
    \centering
    \includegraphics[width=.95\linewidth]{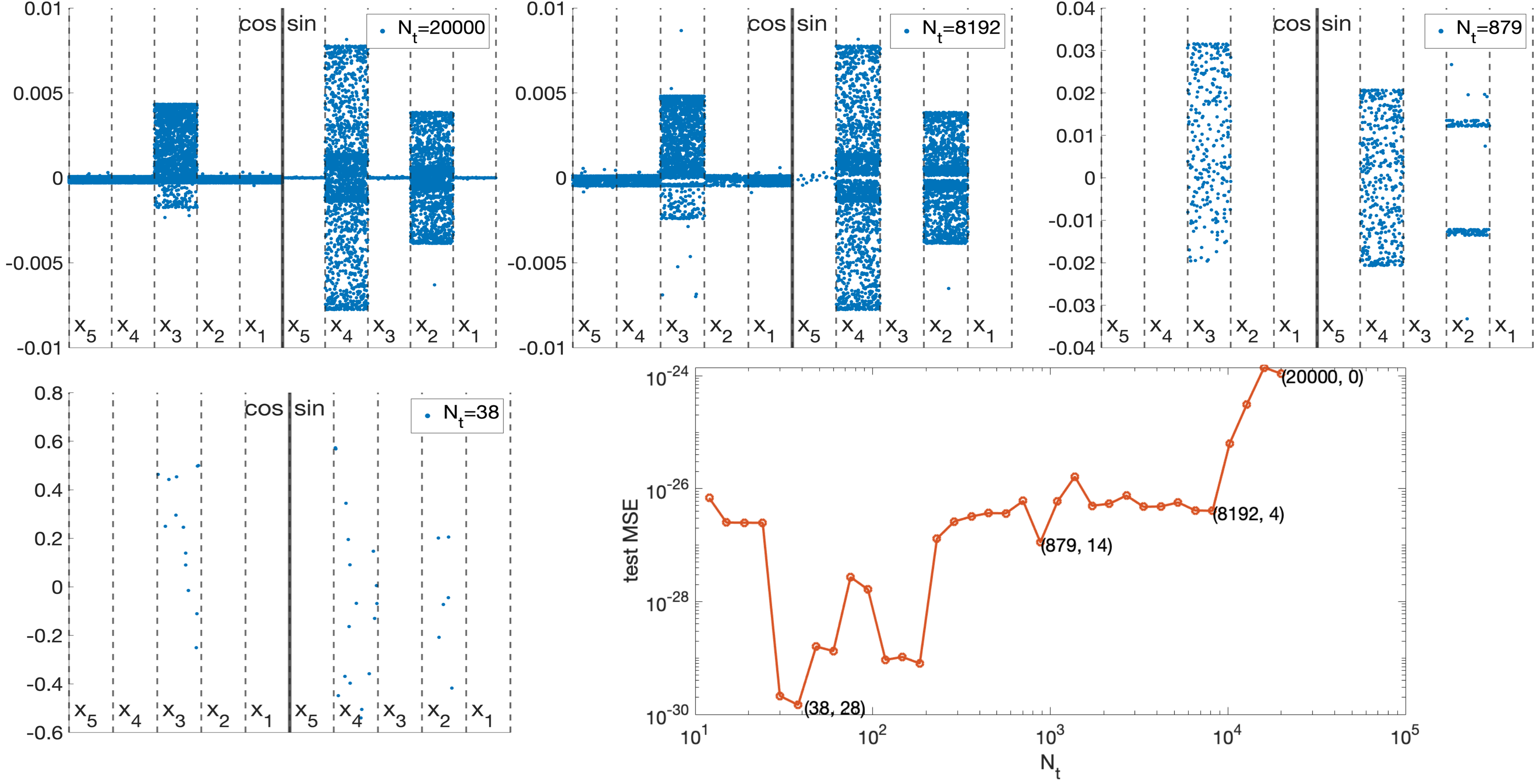}
    \caption{Illustration of support recovery of SHRIMP with $f_s(\vx) = 3\cos(x_3) + 4\sin(x_4) + 2\sin(x_2)$. Top and bottom left: Support recovery plots with $N_t = 20000~(t=0), 8192~(t=4), 879~(t=14), 38~(t=28)$ with $y$-axis corresponding to the weight magnitude; Bottom right: Test MSE tracking of the pruning from $N_t=20000$ to $N_t=12$ with pruning iteration $t \in [0, 33]$.}
    \label{fig:support}
\end{figure}

\textbf{Benefits of Iterative Magnitude Pruning Compared to Random Pruning.} We explore the role of IMP in sparse random feature models by showing the test MSE curves for approximating functions $\{f_2(\vx), f_5(\vx), f_7(\vx)\}$ as defined above, using different number of features $N_t$ (see Figure \ref{fig:role}). We train and evaluate models with SHRIMP, minimal $\ell_2$- and $\ell_1$-norm (SRFE-S \cite{hashemi2021generalization}) estimators with the same $\{N_t\}$ set. Figure \ref{fig:role} illustrates the role and benefit of IMP in finding sparse winning subnetworks. SHRIMP (in blue) has a similar computational cost compared to plain min-$\ell_2$ (in orange), the only difference being in sorting and comparing the absolute weights; at the same time, SHRIMP achieves better test error with a sparser resulting model (i.e., lower $N_t$). The resulting sparse subnetwork behaves better than even the overparameterized solution of plain min $\ell_2$ in the middle plot of $f_5(\vx)$, which also shows the double descent curve. SRFE-S is inefficient due to the computation of $\ell_1$ minimization and is comparably flatter than the other two models, which indicates that it does not benefit from a smaller $N_t$ solution. Hence, the pruning by IMP is efficient and obtains sparser and better subnetworks than the models obtained by $\ell_1$ and $\ell_2$ regularization.
\begin{figure*}[ht]
    \centering
     \includegraphics[width=.32\linewidth]{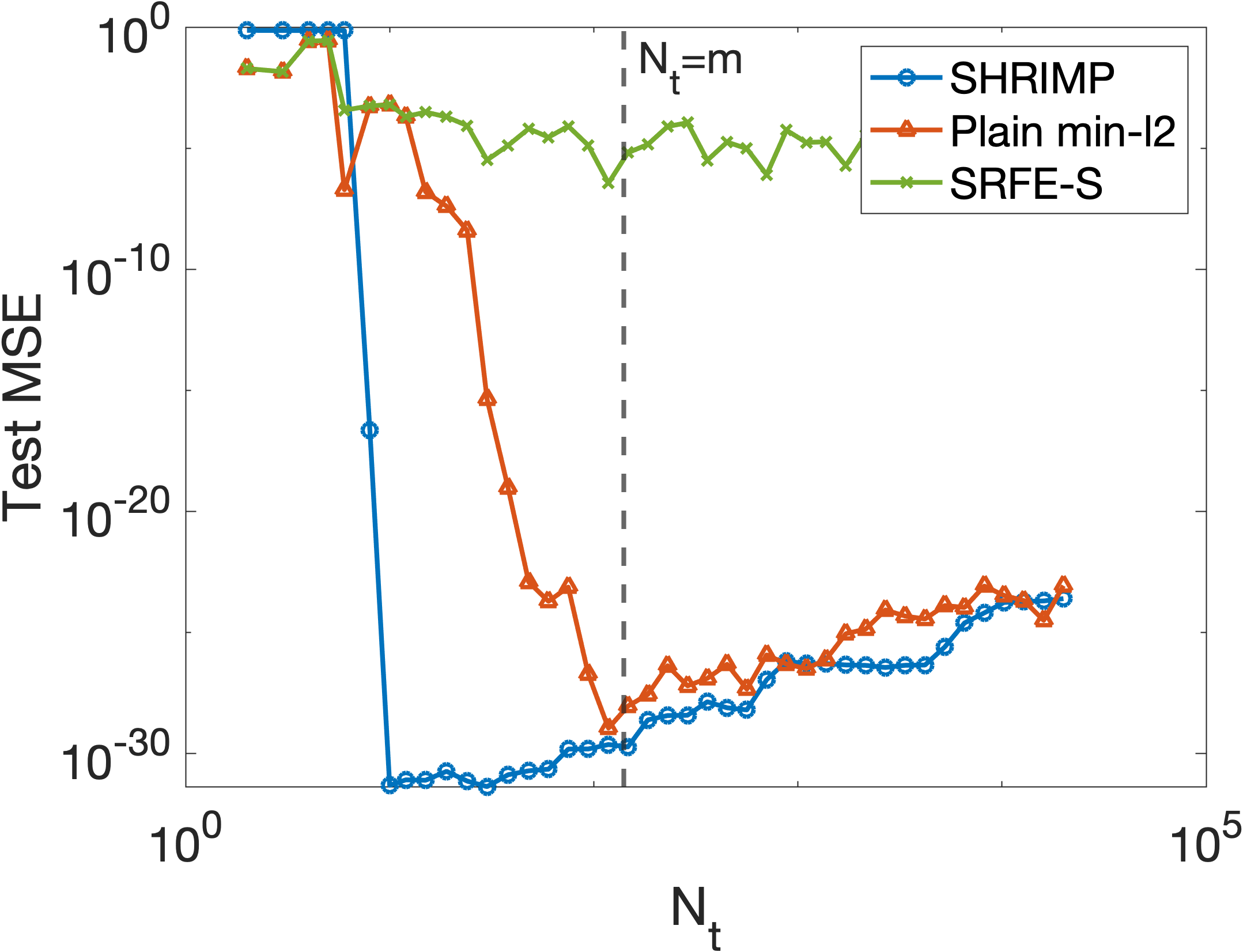}
    \includegraphics[width=.32\linewidth]{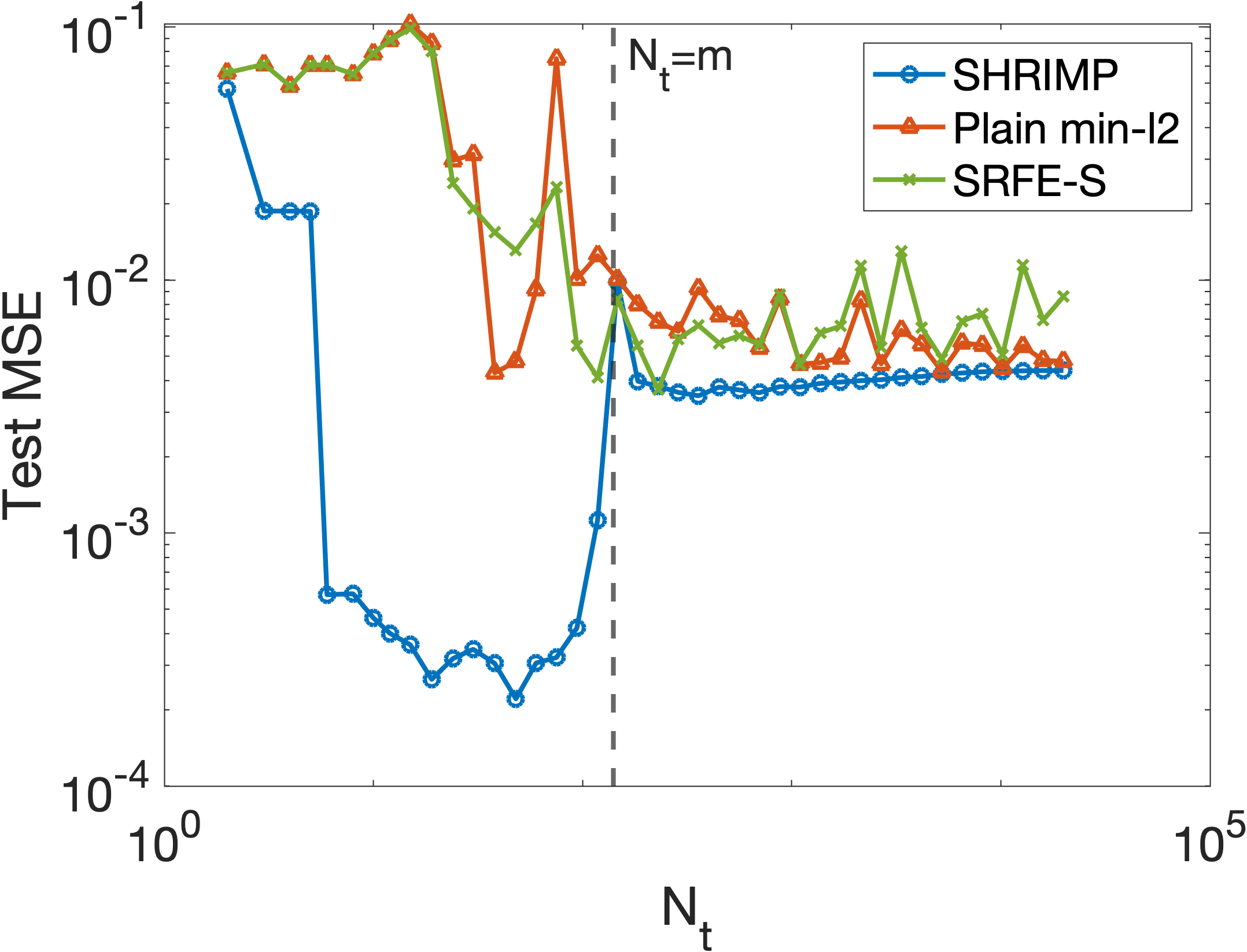}
    \includegraphics[width=.32\linewidth]{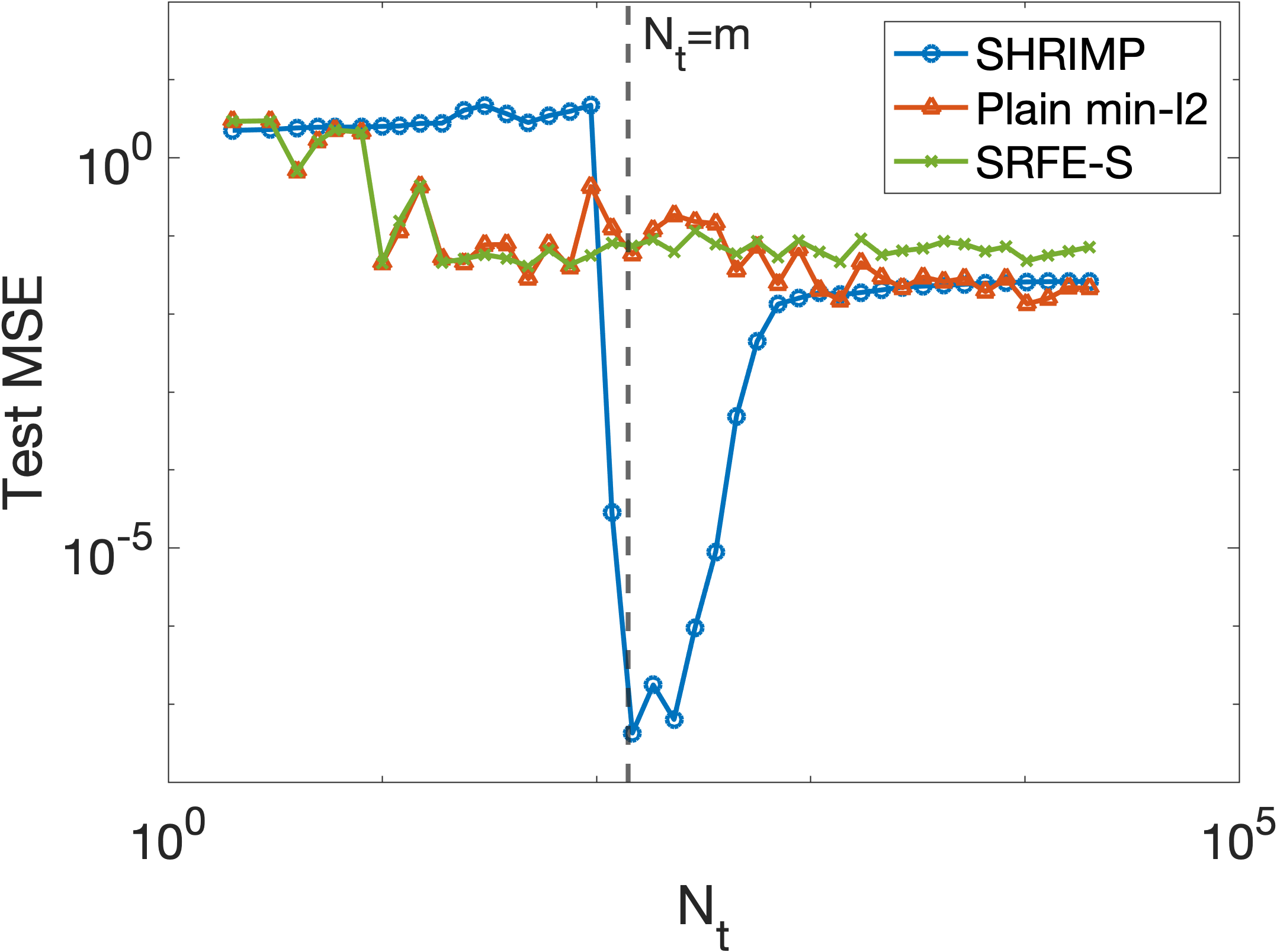}
    \caption{Test MSE of sparse random feature models obtained by SHRIMP, min $\ell_2$-norm estimation, min $\ell_1$-norm estimation (SRFE-S). From left to right: $f_2(\vx), f_5(\vx), f_7(\vx)$.}
    \label{fig:role}
\end{figure*}

\textbf{Spectrum of SHRIMP pruning compared to Random Pruning.} Figure \ref{exp:spectrum_min_max} shows the maximal and minimal eigenvalues
of $A_SA_S^\top /N_t$ throughout pruning for the function $f_7(\vx)=\cos(x_1)x_3+x_2^2x_4+\sum_{j=3}^d x_j$ (Note that we observe similar spectrum patterns for other kinds of functions as well). We notice that for all methods excluding SHRIMP use variance $1/q$---which results in the best generalization error over all cases---these values are essentially constant (up to numerical instability for small $N_t$).  The case for small variance is predicted by the random features approximation of \citet{rahimi_random_2008}, as the kernel approximation is good, while the case for high variance is predicted by \cite{hashemi2021generalization}, where mutual coherence is low.  However, SHRIMP with low variance has a decreasing maximum eigenvalue throughout the pruning process, providing some explanation for the good performance of SHRIMP (with Theorem \ref{thm:refine}); it is important to both perform magnitude pruning and choose a proper variance, as SHRIMP is a two-stage procedure.

\begin{figure}[ht]
    \centering
    \includegraphics[width=.49\linewidth]{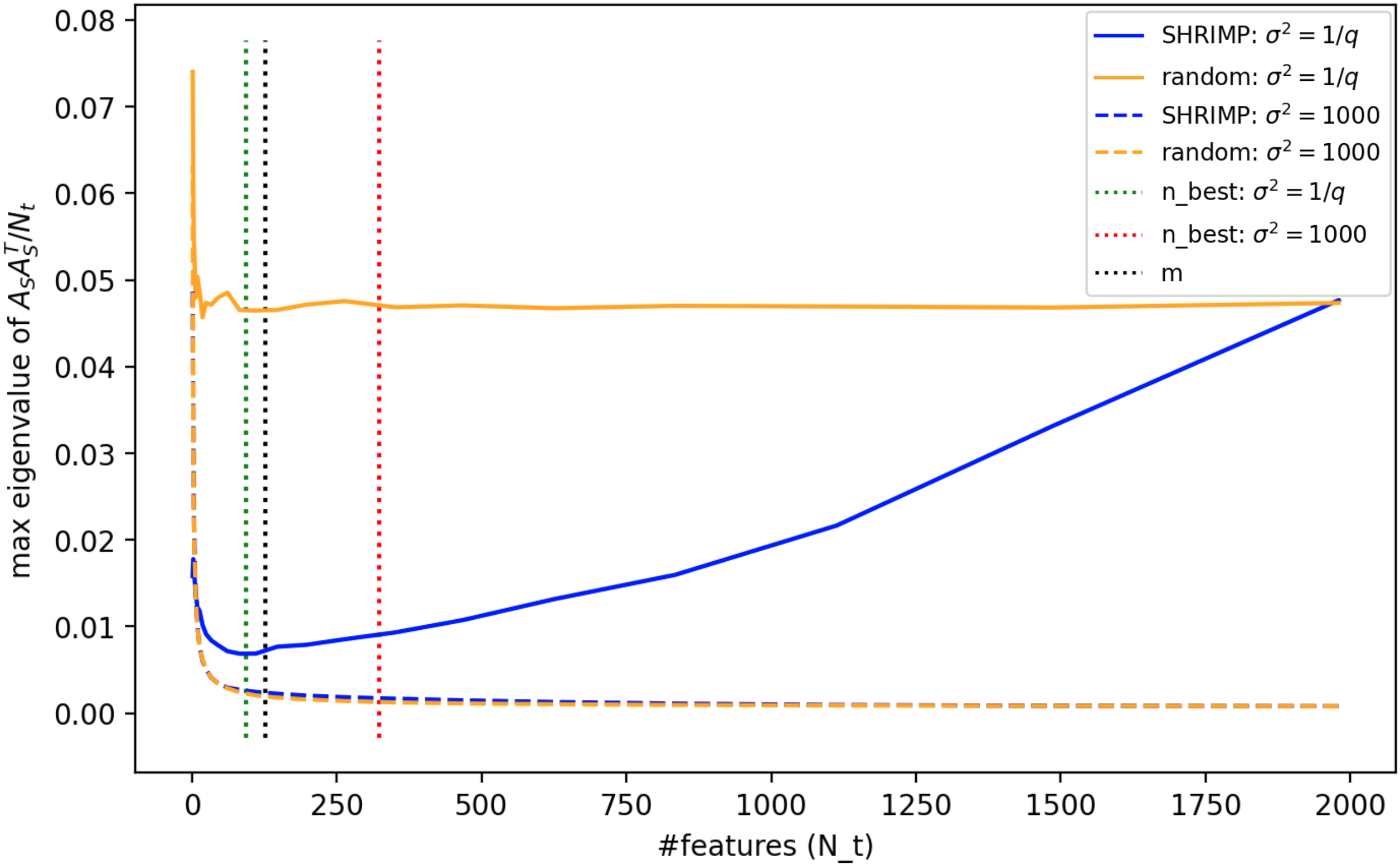}
     \includegraphics[width=.498\linewidth]{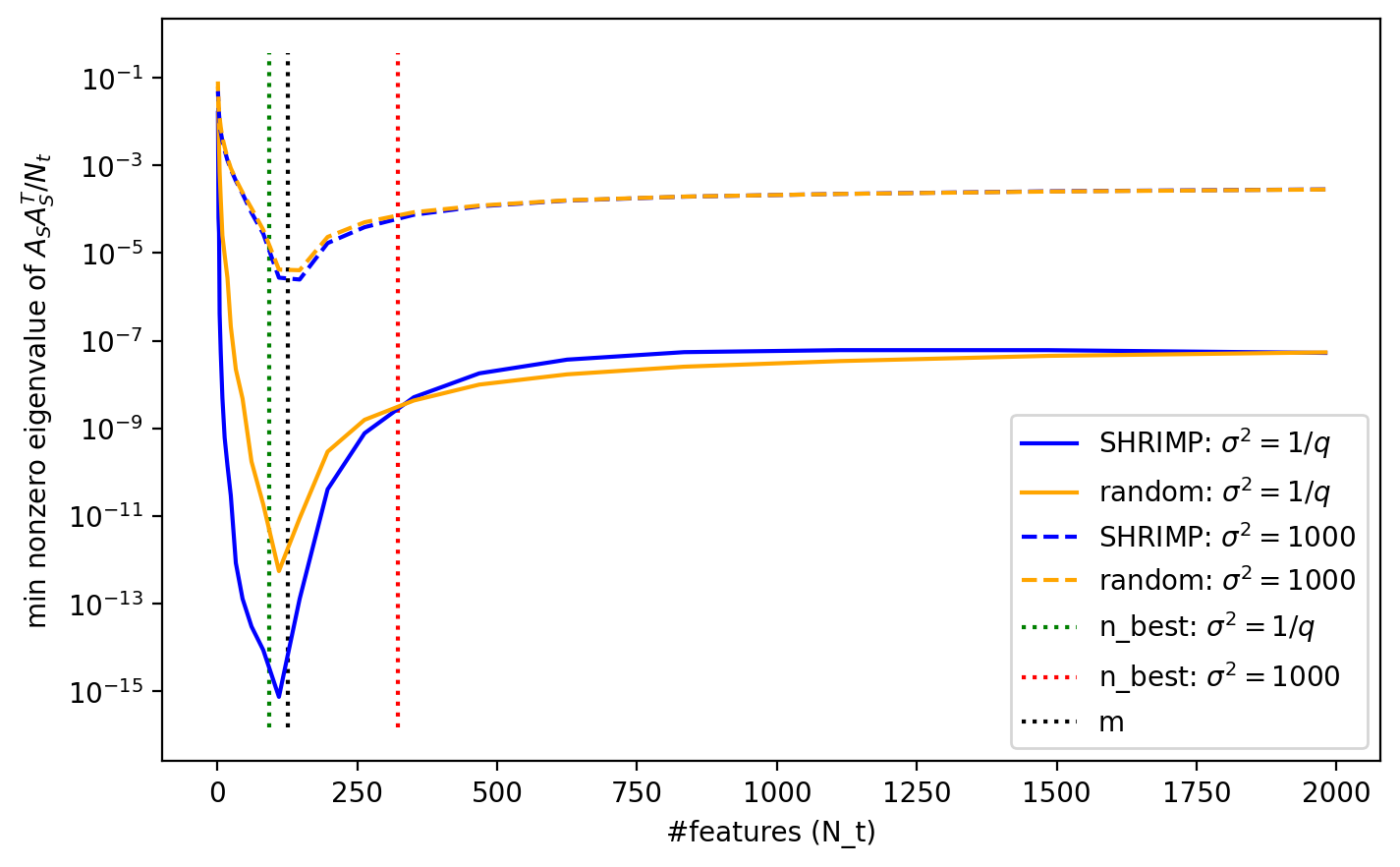}
    \caption{Maximal and minimal eigenvalue of $A_SA_S^\top/N_t$ ($A_S^\top A_S/N_t$ in the underparameterized setting) with weight vectors $\vw \sim \gN(\mathbf{0}, \sigma^2 \mI)$ for  $f_{a_1}(\vx)$ (i.e., $f_7(\vx)$). Blue: SHRIMP; Orange: Random pruning. Test MSE for low variance $\sigma^2 = 1/q$ (solid): SHRIMP(1.19e-05); Random pruning (0.023). Test MSE for high variance $\sigma^2 = 1000$ (dashed): SHRIMP (4.52); Random pruning (4.07).}
    \label{exp:spectrum_min_max}
\end{figure}
\section{Theoretical Analysis}\label{sec:theory} 
In this section, we first provide Theorem \ref{thm:refine}, improving the analysis of the generalization error for thresholded Basis Pursuit from  \cite{hashemi2021generalization}.  Thresholded Basis Pursuit performs basis pursuit followed by a pruning step, keeping only the top $s$ entries of the resulting coefficient vector.  Our analysis refines the result of \cite{hashemi2021generalization} by exposing the role of the maximum singular value of $\mA$ in the resulting generalization bound.  Moreover, we remove the explicit dependence on the number of features $N$, demonstrating that a smaller maximum singular value indicates better generalization. For the proofs of all statements in this section, we refer the reader to the appendix. 

For sake of comparing with SRFE-S in \cite{hashemi2021generalization}, we restate SRFE-S according to our two-stage paradigm in Definition \ref{def:srfe}.
\begin{definition}[Sparse Random Feature Expansion with Sparse Feature Weights (SRFE-S)] With the same input as Algorithm \ref{imp} and a statbility parameter $\eta$, SRFE-S constructs a random feature matrix $\mA$ following Stage I in Algorithm \ref{imp} and solves 
\begin{align}
    \vcsharp = \arg \min_{\vc} \|\vc\|_1 \quad s.t. \quad \|\mA \vc -\vy\| \leq \eta \sqrt{m}
\end{align}
in Stage II. The resulting pruned estimator $\vcsharp|_{\ssharp} $ by SRFE-S keeps the $s$ largest (in magnitude) coefficients on the support set $\ssharp$ and sets $\vcsharp_j = 0, \forall j \in [N] \setminus \ssharp$.
\label{def:srfe}
\end{definition}

\begin{theorem}[\textbf{Generalization Bounds for Thresholded Basis Pursuit}]\label{thm:refine} For a bounded $\rho$-norm function $f$ as defined in Def. \ref{def:function_class1}, construct the dictionary matrix $\mA$ from Stage I in Algorithm \ref{imp} with $m$ samples $\{ (\vx_k, y_k)\}_{k=1}^m$, where $\vx_k \sim \gN(\mathbf{0}, \gamma^2\mI_d)$, $y_k = f(\vx_k)+e_k$ with $|e_k| \leq 2\nu$ or $e_k \sim \gN(0, \nu^2)$, and $\vomega \sim \gN(\mathbf{0}, \sigma^2\mI_d)$, and $\phi(\vx;\vomega)=\exp(i\langle \vx, \vomega \rangle)$. Assume the conditions the following conditions:  $\gamma^2\sigma^2 \geq \frac{1}{2}\left( \left( \frac{\sqrt{41}(2s-1)}{2} \right)^\frac{2}{d} -1 \right)$, number of features $N=\frac{4}{\epsilon^2}\p{1 + 4\gamma\sigma d \sqrt{1+\sqrt{\frac{12}{d}\log \frac{m}{\delta}} } + \sqrt{\frac{1}{2} \log \frac{1}{\delta} } } ^2 $, and number of measurements $m \geq 4 (2\gamma^2 \sigma^2 +1)^d \log \frac{N^2}{\delta}$. Suppose $\fsp$ is estimated by BP (i.e., min $\ell_1$-norm estimator) with $\eta = \min \{ \eta', \teta\}$, where $\eta' = \sqrt{2(\epsilon^2 \norm{f}_\rho^2 + 4\nu^2)} + \sqrt{\frac{\lambda_{\max}(\mA^{*} \mA)}{m}} \kappa_{s,2} (\vcstar)$ and $\teta = 2\sqrt{\epsilon^2\|f\|_{\rho}^2  + 2
\nu^2 + \kappa^2_{s,1}(\vcstar)}$ with approximation error $\epsilon$ and  $\kappa_{s,p}(\vc) := \min\{\norm{\vc-\vz}_{\ell_p}: \vz~ \text{is s-sparse} \}$. Apply an additional pruned step with sparsity $s$, then with probability at least $1-5\delta$, the generalization error is bounded by 
\begin{align}
\begin{split}
    \sqrt{ \int_{\sR^d} \abs{ \fsp (\vx) - \fstar(\vx) } ^2 d\mu }   
    \leq \p{ \frac{8}{m} \log \p{ \frac{1}{\delta} } } ^{\frac{1}{4}}  \p{  2s \| \vcstars - \vcp \|_2^2  + (\kappa_{s, 1}(\vcstar) )^2 } ^{\frac{1}{2}}  + 2 \norm{ \vcp - \vcstars }_2 + \kappa_{s, 1}(\vcstar),
\end{split}
\end{align}
where
\begin{align*}
\| \vcstars & - \vcp \|_2   \leq 
     C \min \bigg\{ 2\sqrt{\epsilon^2\|f\|_{\rho}^2  + 2
        \nu^2 + \kappa^2_{s,1}(\vcstar)},  
     \sqrt{2(\epsilon^2 \norm{f}_\rho^2 + 4\nu^2)} + \sqrt{\frac{\lambda_{\max}(\mA^\ast \mA)}{m}} \epsilon \|f\|_{\rho}
    \bigg\},
\end{align*}
and $\kappa_{s,1}(\vcstar) $ is bounded by $\frac{N-s}{N}\|f\|_{\rho}$.

\end{theorem}

Theorem \ref{thm:refine} connects the numerical results on the decaying maximum singular value of $\mA$, the empirical success of SHRIMP, and refined theory from a similar setting (finding a sparse coefficient vector in a low-order random features model that has good generalization error).  Moreover, in our new analysis, the role of the norm of the smallest entries of the coefficient vector is explicitly revealed: the smaller the norm of the vector of small entries, the better the implied generalization.  This result is directly connected to \citet{belkin2020two}.  Thus, we connect the $\ell_1$-based methods of previous work with $\ell_2$-based methods through this new analysis.

\begin{corollary}[\textbf{Generalization Bounds for Order-$q$ Functions}] \label{cor:sparse}  Fix $\epsilon > 0$. For an order-$q$ function as in Def. \ref{def:order_q_func} with at most $K$ terms, and fix the sparsity $s = n K$ with $N=n \binom{d}{q}$ and $K \ll \binom{d}{q}$. Assume the following conditions: $\gamma^2\sigma^2 \geq \frac{1}{2}\left( \left( \frac{\sqrt{41}(2s-1)}{2} \right)^\frac{2}{q} -1 \right)$, number of features $N=\frac{4}{\epsilon^2}\p{1 + 4\gamma\sigma d \sqrt{1+\sqrt{\frac{12}{d}\log \frac{m}{\delta}} } + \sqrt{\frac{q}{2} \log \frac{d}{\delta} } } ^2 $, and number of measurements $m \geq 4 (2\gamma^2 \sigma^2 +1)^{\max\{2q-d, 0\}}  (\gamma^2 \sigma^2 +1)^{\min\{2q, 2d-2q\}} \log \frac{N^2}{\delta}$. Then the generalization error corresponding to the thresholded $\ell_1$ estimator with the $s$ largest elements (in magnitude) is bounded by  $ \gO \p{ \p{ 1 + C's^{\frac{1}{2}}m^{-\frac{1}{4}} \log^{\frac{1}{4}}(\frac{1}{\delta}) } \sqrt{\epsilon^2 \binom{d}{q} \tvert{f}^2 + E^2 } } $ 
with probability at least $1-5\delta$, where $\tvert{f} := \frac{1}{K}\sum_{j=1}^K\|g_j\|_{\rho}$.
\end{corollary}

Note that Corollary \ref{cor:sparse} improves the generalization bound of SRFE-S (Def. \ref{def:srfe}) in \cite{hashemi2021generalization} from depending on the number of features $N$ to the sparsity level $s$ in the $1 + C's^{\frac{1}{2}}m^{-\frac{1}{4}} \log^{\frac{1}{4}}(\frac{1}{\delta})$ term.

Then we shed light on the maximal and minimal eigenvalues of the Gram matrices of SHRIMP observed in Figures \ref{exp:spectrum_min_max}.  The bounds in Proposition \ref{prop:spectrum} below are obtained using techniques inspired by \cite{chen2021conditioning}.

\begin{proposition}[\textbf{Bounds on Eigenvalues of Gram Matrix}]\label{prop:spectrum} 
Consider data $\{\vx_1,\dots, \vx_m\}$ i.i.d. drawn from $\mathcal{N}(0,\gamma^2 \mI_d)$, weights $\{\vomega_1, \dots, \vomega_N\}$ i.i.d. drawn from $\mathcal{N}(0,\sigma^2 \mI_q)$, and the Fourier feature matrix $a_{j,k} = \phi(\vx_j,\vomega_k)$, where $\phi(\vx,\vomega) = \exp(i \langle \vx ,\vomega \rangle)$. Fix the feature sparsity $q \leq d$ as in Def. \ref{def:order_q_func} and consider the regime $m\leq N$. Let $\lambda_k(\frac{1}{N} \mA \mA^\ast)$ be the $k$th eigenvalue of the scaled Gram matrix. Then the expectation of the maximum eigenvalue $\lambda_1$ and the minimum eigenvalue $\lambda_m$ of the matrix $\frac{1}{N}\mA\mA^\ast$ satisfy
\begin{align}
     \mathbb{E}\lambda_1 & \geq 2 - \frac{(N-1)m}{N^2}+ \frac{(N-1)(m^2-m)}{N^2}\left(4\gamma^2\sigma^2+1\right)^{-\frac{q}{4}}, \\
    \mathbb{E} \lambda_{m}
   &\leq 
   \frac{c-1}{c}+ \frac{1}{m} +\left( \frac{c-1}{c} m + 1 \right) \left(4\gamma^2\sigma^2+1\right)^{-\frac{q}{4}},
\end{align}
where $c=N/m$.
\end{proposition}

\begin{remark}
 Using Markov's inequality, for  $c\rightarrow 1^+$, i.e. $N=m$,  we have 
 \begin{align}
  \lambda_{m} \left(\frac{1}{N}A A^\ast\right) \leq N^{p} \left(4\gamma^2\sigma^2+1\right)^{-\frac{q}{4}} + N^{p-1}
  \end{align}
  with probability $1-N^{-p}$ for $0<p<1$. If $\gamma^2\sigma^2 = \mathcal{O}(1)$, then we observe the benefit of small $q$ (i.e., low-order interactions); if $q \rightarrow d$ (in the high dimensional setting), then the minimum eigenvalue becomes arbitrarily small while the maximum remains above 2, and thus the system is ill-conditioned. In particular, the conditioning is directly related to the size of $\left(4\gamma^2\sigma^2+1\right)^{-\frac{q}{4}}$.
\end{remark}

\textbf{Further Discussion.} We connect our results to $\ell_0$-based methods.  First we note the explicit connection to SINDy \citep{zhang2019convergence}, which algorithmically is similar to SHRIMP, except instead of pruning the smallest magnitude coefficients, it prunes the all entries smaller than some threshold.  However, their results are about coefficient recovery, and generalization bounds for $\ell_0$-based methods are sparse in the literature.  From Remark 2 in \citet{nikolova2013description}, the iterates of SHRIMP are \textit{each} local minimizers of an $\ell_0$-regularized problem, which gives insight to the behavior of SHRIMP; However, SHRIMP arrives in an adaptive and greedy nature, depending on the solution of the previous one. Further discussion on these topics is given in the appendix.

\section{Discussion}\label{sec:dis}
We propose a new method, \textbf{S}parser \textbf{R}andom Feature Models with \textbf{IMP}, to exploit low-order additive structure in a learning problem, which often occurs in many domains of interest.  In this method, we explicitly construct a sufficiently overparameterized sparse feature matrix in order to approximate a given underlying low-order function, and then prune coefficients by adaptively solving a min $\ell_2$-norm problem and applying iterative magnitude pruning. This can be seen as an instance of feature selection or neuron pruning in the neural network pruning literature. We test our method on both synthetic and real datasets: SHRIMP vastly exceeds other methods on synthetic data; and it is often better or at least competitive on real datasets. We illustrate the relationship between low-order structure and pruning, corresponding to weight and neuron pruning, respectively, and show the IMP has the greatest effect when combined with sparse feature models. Our analysis provides generalization bounds for thresholded BP and bounds on eigenvalues of Gram Matrix, which explains the benefits of our method. We hope to shed some light on the lottery ticket hypothesis in a simple model, similar to how regression is once again being studied in the context of deep learning theory; our method corresponds to certain pruning methods in two-layer neural networks.

More robust generalization bounds can be given to our SHRIMP model---for example, in the context of random features regression, studying the eigenspectrum of a pruned sub-Gram matrix throughout our algorithm is a possible extension of our work. Another possible future direction is to adaptively discover the low-order structure as we go, instead of fixing the parameter $q$ in advance.  This results in a setting more closely tied to practical pruning and allows for greater flexibility (e.g., if the underlying function is a sum of functions of various orders), and may shed light on what a pruned network is learning.

\section*{Acknowledgments}
B. Shi, R. Ward, and Y. Xie were supported in part by AFOSR MURI FA9550-19-1-0005, NSF DMS 1952735, NSF HDR-1934932, and NSF 2019844. H. Schaeffer was supported in part by AFOSR MURI FA9550-21-1-0084 and NSF DMS-1752116. 

\bibliography{sparse}
\onecolumn
\appendix
\section*{Appendix}
The appendix is organized as follows:
\begin{itemize}
    \item Appendix \ref{app:setup}: Experimental Details of Function Approximation 
    \item Appendix \ref{app:exp_more}: Additional Experiments
    \item Appendix \ref{app:prop}: Proofs of Theorems 
    in Section \ref{sec:theory}
 
    \item Appendix \ref{app:further}: Further Discussion

\end{itemize}

\section{Experimental Details of Function Approximation} \label{app:setup}
\textbf{Data Generation.} We generate the data $ \mX \times \vy \in\mathbb{R}^{m\times d} \times \mathbb{R}^m$ in the following way: (1) sample $m$ $d$-dimensional points $\vx_1,\dots, \vx_m$ with $\vx_j\sim \mathrm{Unif}[-1, 1]^d$, except for the Ishigami function \footnote{It is the traditional sampling way for Ishigami function.} from $\mathrm{Unif}[-\pi, \pi]^d$; (2) For each, function $f_i$, $y_j=f_i(\vx_j)$. Note that we include no additive noise in our experiments, although the observed behavior is robust to the presence of noise.  

\textbf{Experimental Set-up.} In the function approximation experiments, models are evaluated in both low-dimensional ($s_l^{q_*}$ and $s_l^{d}$ with $m=140, d=10$) and high-dimensional ($s_h^{q_*}$ and $s_h^{d}$ with $m=1400, d=100$) settings. The results for $q=q_*$ are in Table \ref{tab:compare_func}, and the full results with both sparse ($q=q_*$) and dense features ($q=d$) are in Table \ref{tab:full} in the appendix. For SRFE-S, Min $\ell_2$, and SHRIMP in $s_l^{q_*}$ and $s_h^{q_*}$, we sample $\mathbf{w}$ according to Def. \ref{def:sparse_weights} and $\rho=\mathcal{N}\left(\mathbf{0}, q^{-1}I_q\right)$ with $q=q_*$ as the actual order of those low-order functions.\footnote{In the low-order case, since the actual orders are known and are small enough such that $\binom{d}{q} < N$, we can use the actual order $q_*$.}, We set $n=N/\binom{d}{q}$ with $N=10000$ in our experiments and form the random feature matrix $\mW \in\mathbb{R}^{N\times d}$. The dictionary $\mA=[\cos(\mX \mW^\top ), \sin(\mX \mW^\top )]\in\mathbb{R}^{m\times 2N}$.  For SHRIMP, we set $0.2$ as the pruning rate and validate on $10\%$ of the training set to choose the best pruned model. For SALSA, we form the kernel matrix $\mK$ by
$K_{ij}=\sum_{k=1}^{\binom{d}{q}} \exp\left(-\frac{\norm{\vx_i|_{\mathcal{S}_k}-\vx_j|_{\mathcal{S}_k}}^2}{2q}\right)$;
note that $\mathbb{E}_{W}[\mA\mA^\top ]=\mK$. 

Each model is evaluated by the average of test mean squared errors over three runs. For the sake of completeness, we also experiment with the same functions with $q=d$ for all functions ($s_l^{d}$ and $s_h^{d}$ in Table \ref{tab:compare_func}---which corresponds to standard kernel regression and random feature regression with standard Gaussian kernel. Here, random weights are drawn $\mathbf{\mW}\sim \mathcal{N}\left(\mathbf{0}, d^{-1}I_d\right)$ and fully dense. 
\section{Additional Experiments}\label{app:exp_more}
\subsection{Comparing Function Approximations with Sparse and Dense Features}

Table \ref{tab:full} shows the full results of function approximation with sparse ($q=q_*$) and dense features ($q=d$). As shown in Table \ref{tab:full}, for both the low-dimensional $(s_l^{q_*}, s_l^{d})$ and high-dimensional $(s_h^{q_*}, s_h^{d})$ settings, models with $q=q_*$ (i.e., $s_l^{q_*}$ and $s_h^{q_*}$), where the low order $q$ matches the actual order of functions, have significantly better performance for all methods over corresponding models with $q=d$ (i.e., $s_l^{d}$ and $s_h^{d}$). This shows the benefit of our use of low-order structure compared to previous random features work with dense features. However, with dense features $q=d$, the advantage of pruning over other methods fades: pruning over all functions performs comparably to standard $\ell_1$ and $\ell_2$ based methods in both low and high dimensions since all features add to the representative capacity, and when SHRIMP does perform worse, it is very slight. This is also exhibited as the average size of the model is much larger when using dense features as opposed to sparse features.  

\begin{table}[ht]
    \centering
    \caption{Comparison of the test errors. *Avg size denotes the average pruned model size of SHRIMP over three runs. $^\dagger s$ denotes the setting with different $(m, d, q)$ pairs: $s_l^{q_*} = (140, 10, q_*), s_l^d = (140, 10, d), s_h^{q_*} = (1400, 100, q_*), s_h^{d} = (1400, 100, d)$, where $l$ denotes low dimension and $h$ denotes high dimension. The best MSE for each $f_i(\vx)$ over all models is in purple. In $s_l^{q_*}$ and $s_h^{q_*}$, all the methods use $q=q_*$ (the ground-truth low order) to interpolate the functions and achieve better test performance than $s_l^{d}$ and $s_h^{d}$ (which use dense features), respectively. From the comparison of $s_l^{q_*}$ and $s_h^{q_*}$, our SHRIMP method is scalable to the high-dimensional setting.}

    \scalebox{1}
    {\begin{tabular}{clccccccc}
    \toprule
    Setting$^{\dagger}$ & Model  & $f_1(\vx)$
    & $f_2(\vx)$
    & $f_3(\vx)$
    & $f_4(\vx)$ 
    & $f_5(\vx)$
    & $f_6(\vx)$
    & $f_7(\vx)$
    \\ 
 
    \midrule

    & SRFE-S & 7.85e-04 & 1.98e-05 & 1.15e-01 & 1.27e-01  & 7.52e-03 & 1.11 & 7.49e-02\\
     & Min $\ell_2$ & 4.37e-20 & 5.45e-24 & 8.20e-02 & 5.94e-02 &  7.36e-03 & 7.18 & 2.98e-02\\
    $s_l^{q_*}$ & SALSA & 1.59e-12 & 1.26e-15 & 8.80e-02 & 6.14e-02 &  7.32e-03 & 6.99 & 2.72e-02\\
     & SHRIMP & \textcolor{purple}{1.37e-22} & \textcolor{purple}{7.90e-32} & \textcolor{purple}{4.98e-12} & \textcolor{purple}{2.54e-12}  & \textcolor{purple}{6.39e-04} & \textcolor{purple}{2.58e-02} & \textcolor{purple}{2.83e-05}\\
     \hdashline
    & Avg size* & 3100.33 & 29 & 147 & 171.67 & 39 & 80.33 & 187 \\
    
    \midrule
     & SRFE-S & 5.35e-02 & \textcolor{purple}{1.47e-03} & 5.56e-02 & 1.64e-01 &  \textcolor{purple}{4.11e-03} & \textcolor{purple}{1.27e+01} & 9.17e-02 \\
    & Min $\ell_2$ & 1.71e-02 & 1.93e-03 & 3.50e-02 & 1.01e-01 & 4.62e-03 & 1.45e+01 & 6.15e-02\\
    $s_l^d$ & SALSA & \textcolor{purple}{1.68e-02} & 1.91e-03 & \textcolor{purple}{3.38e-02} & \textcolor{purple}{9.67e-02} &  4.62e-03 & 1.44e+01 & 6.05e-02\\
     & SHRIMP & 1.70e-02 & 1.63e-03 & 3.51e-02 & 9.81e-02 &  6.57e-03 & 1.60e+01 & \textcolor{purple}{5.91e-02}\\
     \hdashline
    & Avg size & 6726 & 77.67 & 1872 & 12020  & 450.33 & 5342.33 & 39.33\\
    
    \midrule
    & SRFE-S & 1.52e-03 & 8.71e-06 & 1.99 & 4.54&  1.16e-01 & 4.59 & 4.40e-01\\
    & Min $\ell_2$ & 1.68e-20 & 3.51e-24 & 2.01 & 4.75 &  1.16e-01 & 8.34 & 1.49e-01\\
    $s_h^{q_*}$& SALSA & 1.99e-11 & 2.54e-13 & 1.60 & 3.16 &  8.49e-02 & 7.35 & 1.34e-01\\
     & SHRIMP & \textcolor{purple}{1.61e-22} & \textcolor{purple}{1.11e-30} & \textcolor{purple}{1.26e-02} & \textcolor{purple}{5.11e-01} & \textcolor{purple}{1.50e-02} & \textcolor{purple}{2.68} & \textcolor{purple}{5.82e-02}\\
     \hdashline
    & Avg size & 3355 & 19 & 61.33 & 64  & 42.67 & 13 & 229\\
    \midrule
    & SRFE-S & 1.43e-01 & 2.35e-02 & \textcolor{purple}{1.55e+00} & \textcolor{purple}{3.05e+00} &  \textcolor{purple}{8.20e-02} & 1.69e+01 & 2.06e-01 \\
    & Min $\ell_2$ & \textcolor{purple}{6.01e-02} & \textcolor{purple}{2.33e-02} & \textcolor{purple}{1.55e+00} & \textcolor{purple}{3.05e+00} &  \textcolor{purple}{8.20e-02} & \textcolor{purple}{1.40e+01} & 8.33e-02\\
    $s_h^{d}$ & SALSA & 2.62e+04 & 1.23e+03 & 4.63e+03 & 3.34e+04 & 3.20e+02 & 2.44e+04 & 4.66e+04\\
     & SHRIMP & 6.07e-02 & 2.35e-02 & \textcolor{purple}{1.55e+00} & 3.07e+00 &  9.33e-02 & 1.41e+01 & \textcolor{purple}{8.32e-02}\\
     \hdashline
    & Avg size & 6663 & 14933.33 & 8199 & 6908.67  & 717 & 14730.67 & 8328.67\\
    \midrule
    & Order $q_*$ & 1 & 1 & 2 & 2 &3 &2 &2 \\
     \bottomrule
    \end{tabular}}
    \label{tab:full}
\end{table}

\subsection{Additional Iterative Magnitude Pruning Curves}

We show comprehensive curves in Figure \ref{fig:more} and \ref{fig:naive} to illustrate the role of IMP in sparse random feature models using the functions defined in Section \ref{sec:exp}  with different number of features $N_t$. For sake of completeness, Figure \ref{fig:more} shows more types of test MSE curves with approximating functions $f_1(\vx), f_3(\vx), f_4(\vx), f_6(\vx)$, which are not included in Section \ref{sec:exp} due to page limit.  Except for $f_1(\vx)$, where the pruned curve has a little better test performance over the best solution of SRFE-S but with more number of features since the coefficient vector of function $f_1(\vx)=\sum_{i=1}^{d-1}x_i+\exp(-x_d)$ is comparably dense with random sparse features, SHRIMP find sparser estimators with better performance than SRFE-S and plain min $\ell_2$-norm estimator on other functions. 
\begin{figure}[ht]
    \centering
    \includegraphics[width=.245\linewidth]{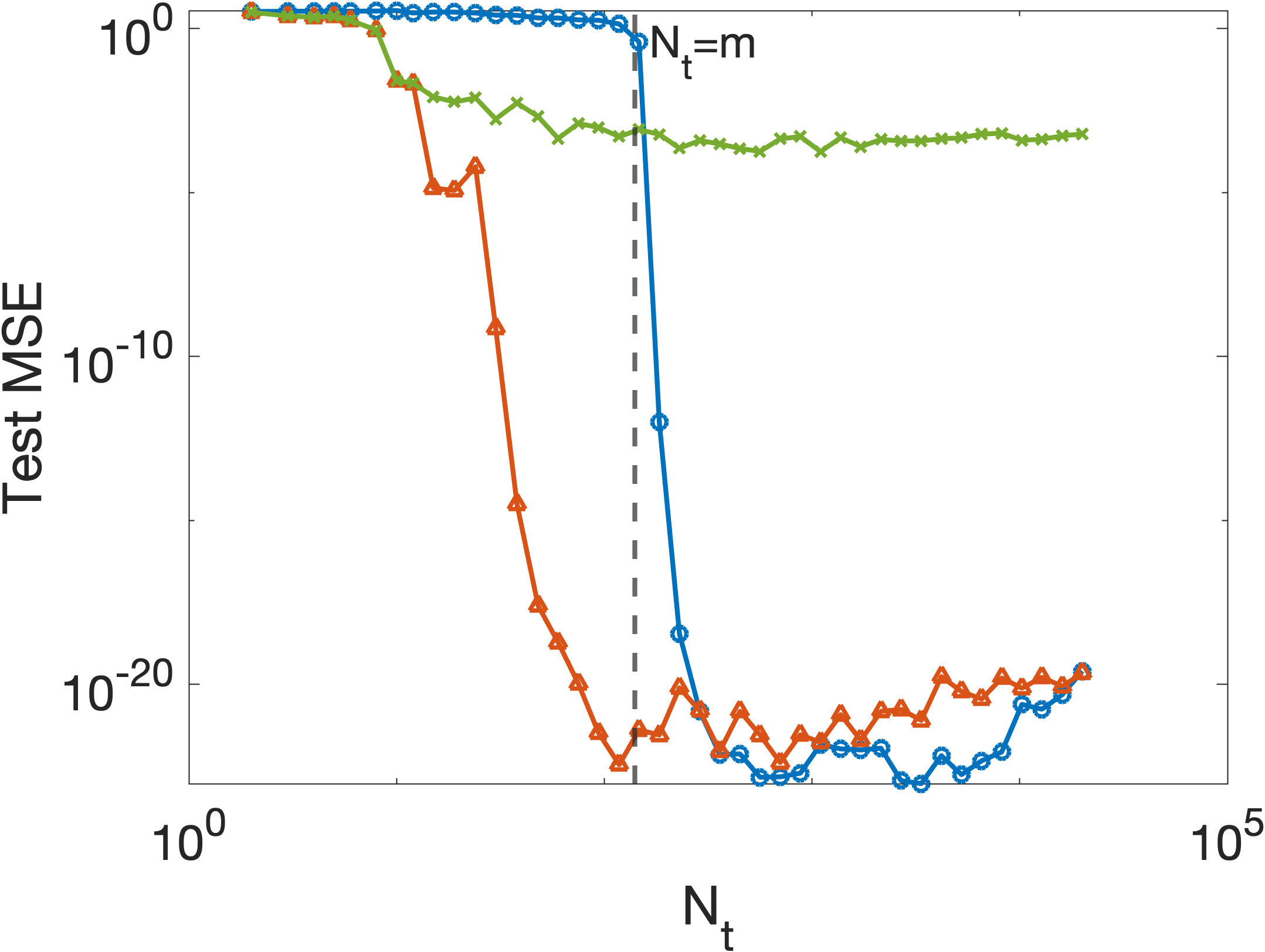}
    \includegraphics[width=.245\linewidth]{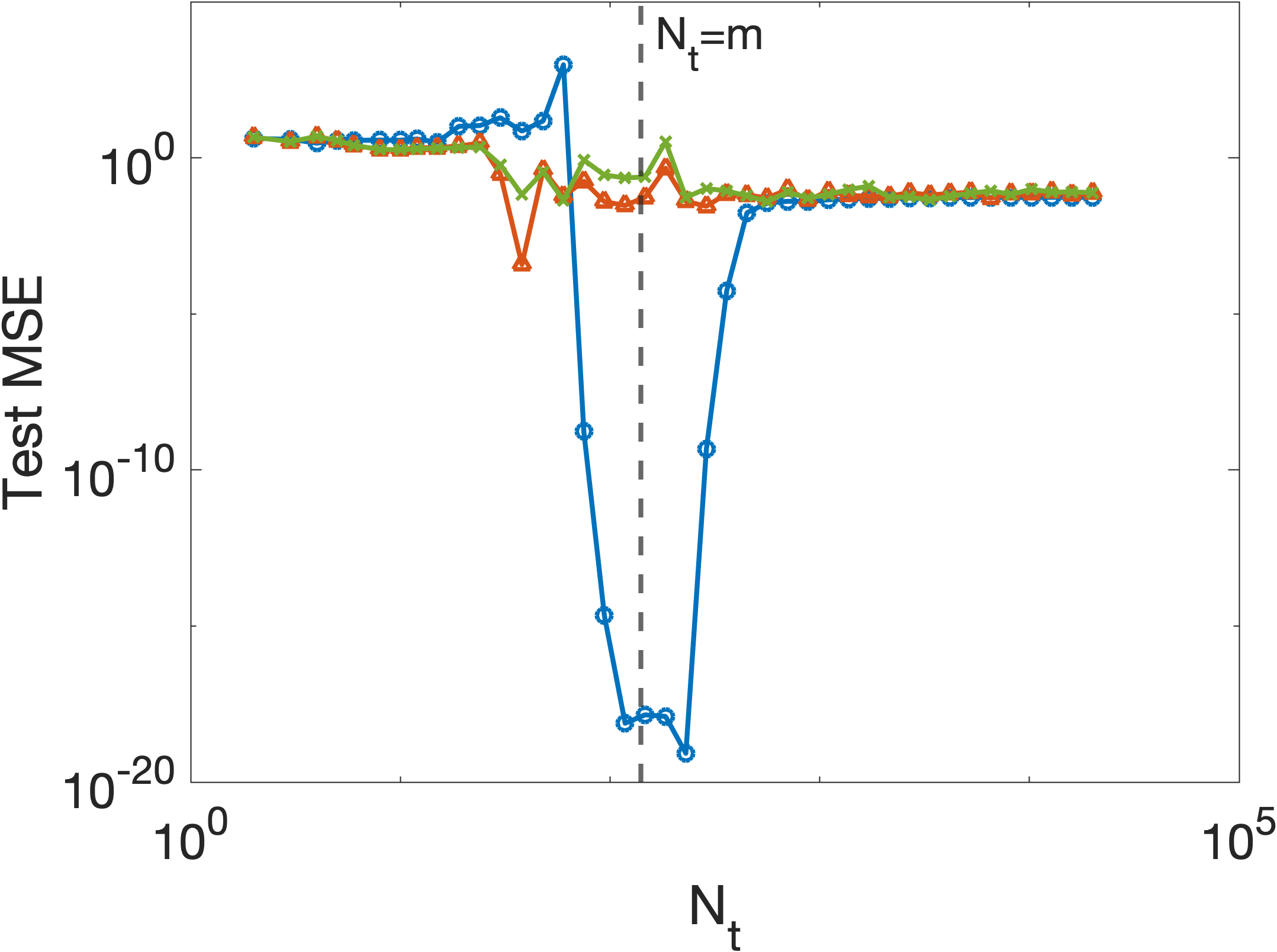}
    \includegraphics[width=.245\linewidth]{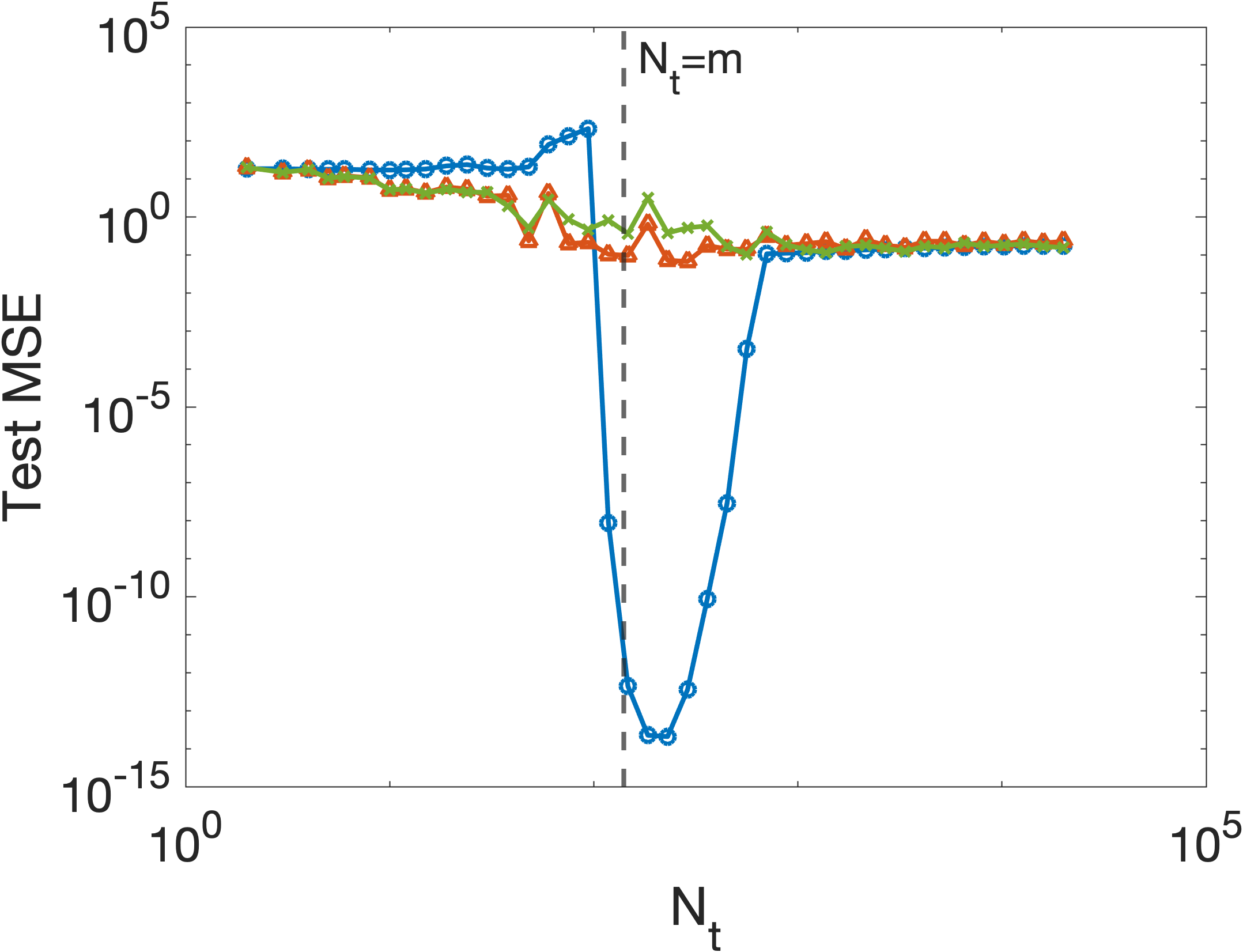}
     \includegraphics[width=.245\linewidth]{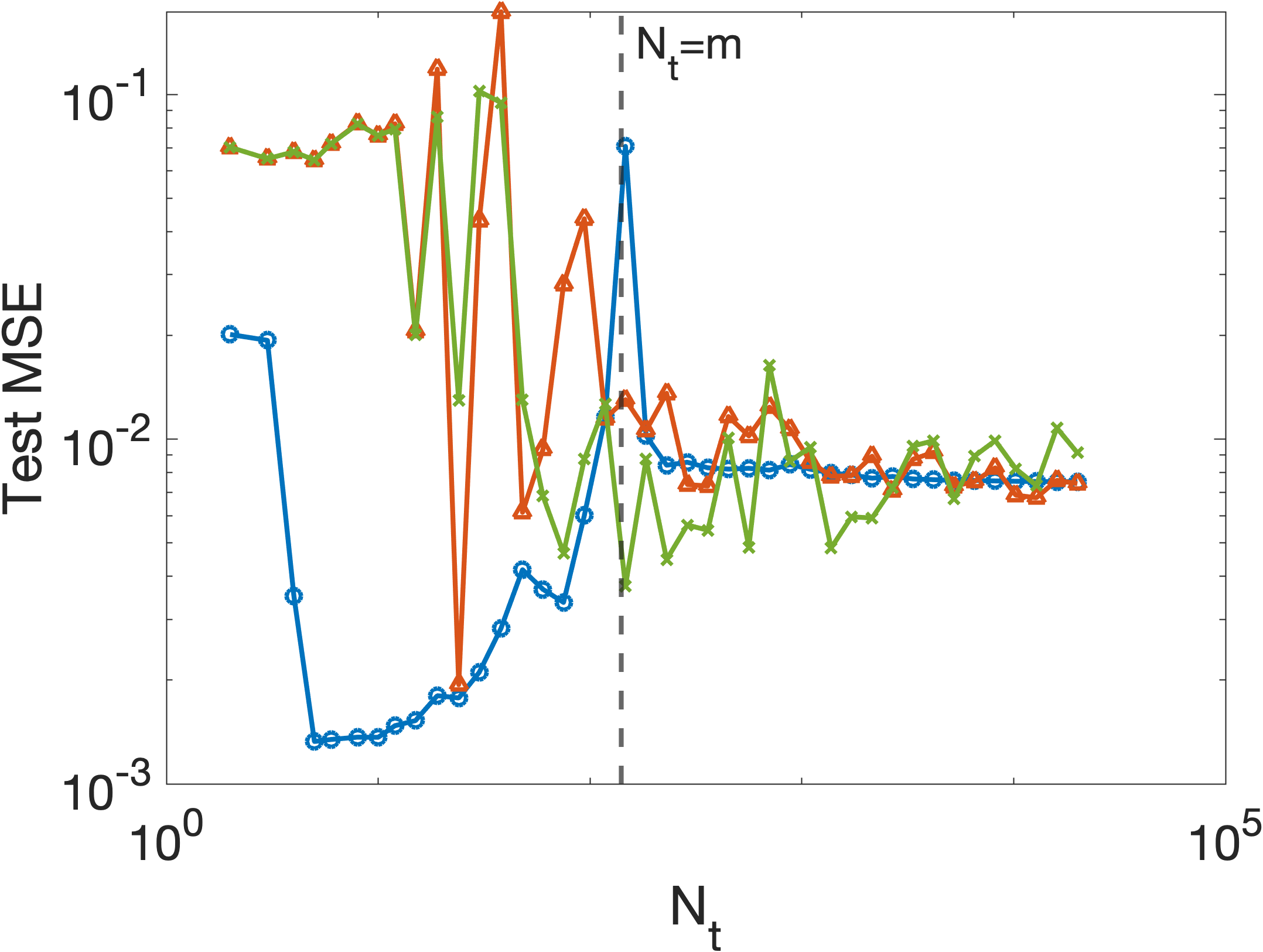}
    \caption{Test MSE of sparse random feature models obtained by SHRIMP, min $\ell_2$-norm estimation, min $\ell_1$-norm estimation (SRFE-S). From left to right: $f_1(\vx), f_3(\vx), f_4(\vx), f_6(\vx)$ as defined in Section \ref{sec:exp}. Blue: SHRIMP; Orange: Plain min $\ell_2$; Green: SRFE-S.}
    \label{fig:more}
\end{figure}

Furthermore, we include additional curves comparing the aforementioned methods to naive pruning methods (where the weights are kept fixed without retraining after each pruning step) in Figure \ref{fig:naive}. As shown in those figures, naive pruning usually has worse performance than the original minimal $\ell_2$-norm solution, let alone SHRIMP, which verifies the benefit of retraining from the same initialization of iterative magnitude pruning as \citet{frankle_lottery_2019} suggests. 
\begin{figure}[ht]
    \centering
    \includegraphics[width=.32\linewidth]{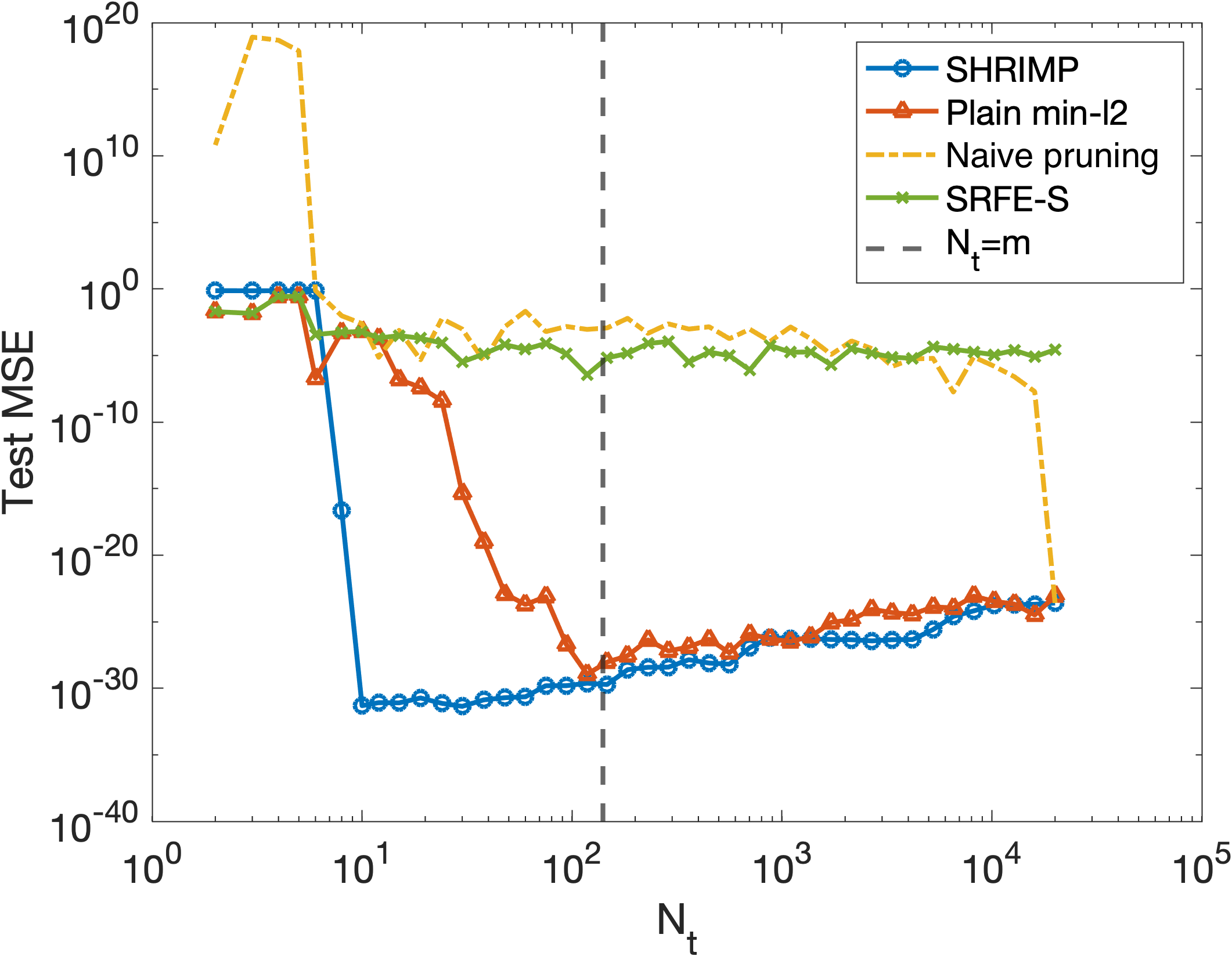}
    \includegraphics[width=.32\linewidth]{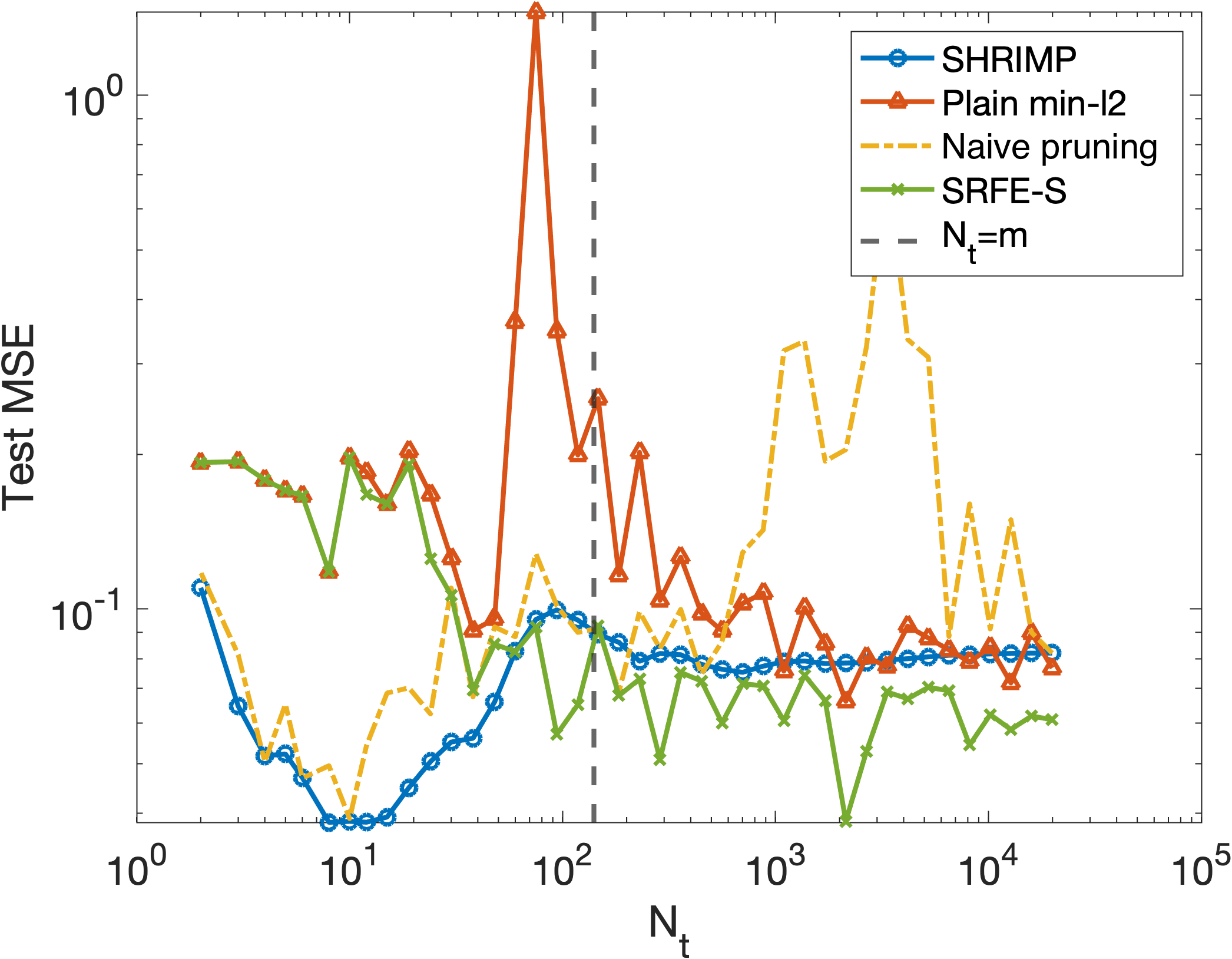}
    \includegraphics[width=.32\linewidth]{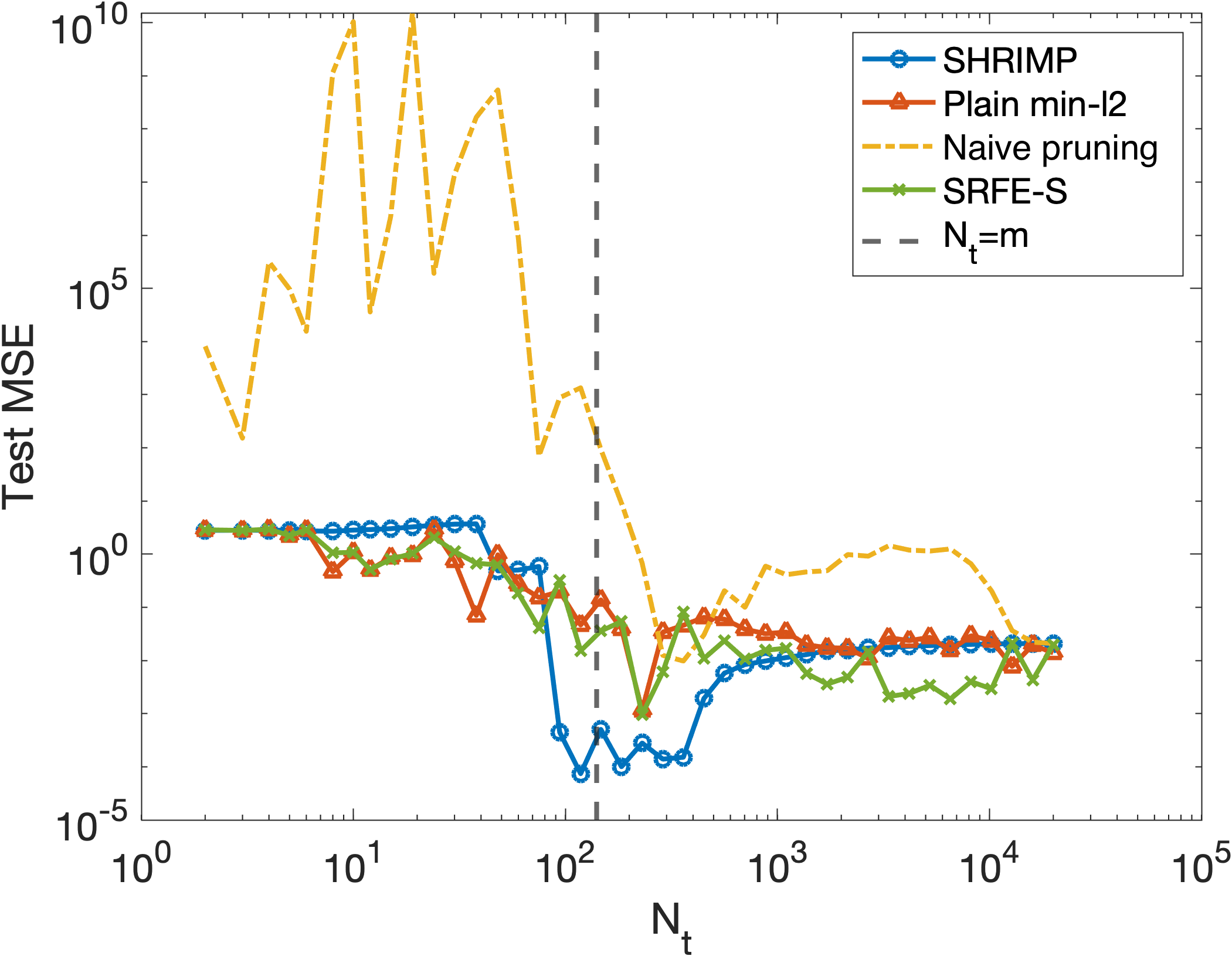}
    \includegraphics[width=.32\linewidth]{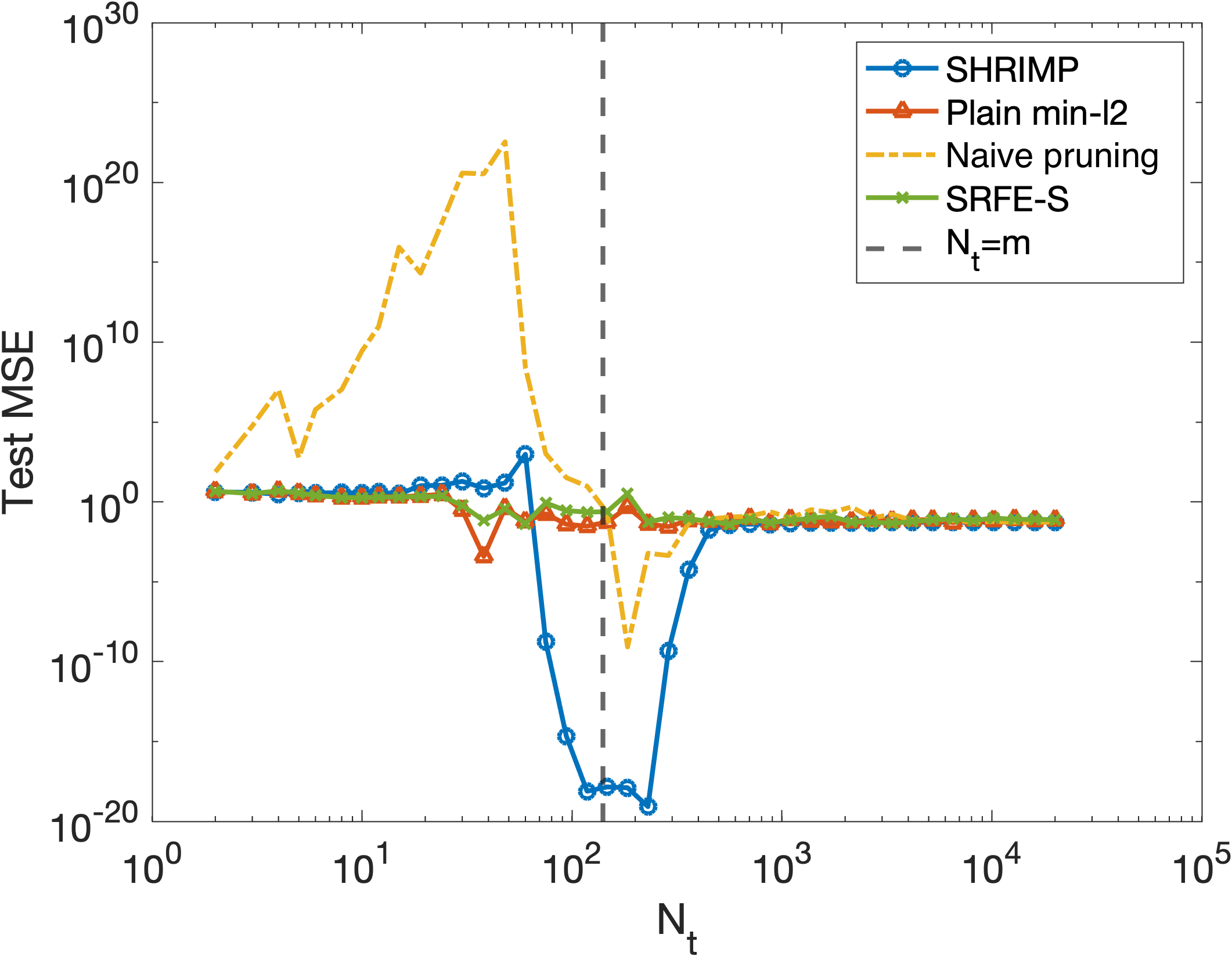}
    \includegraphics[width=.32\linewidth]{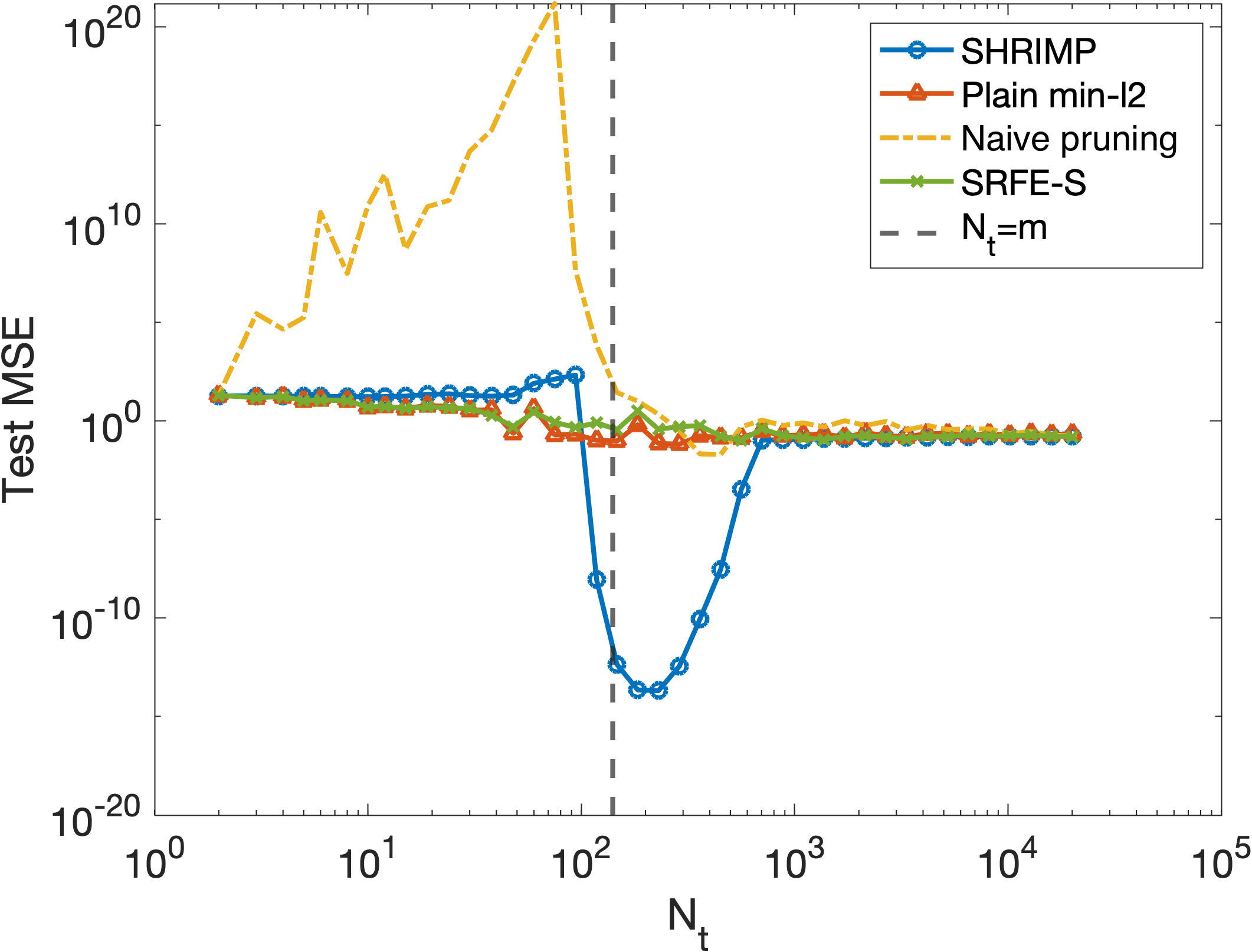}
     \includegraphics[width=.32\linewidth]{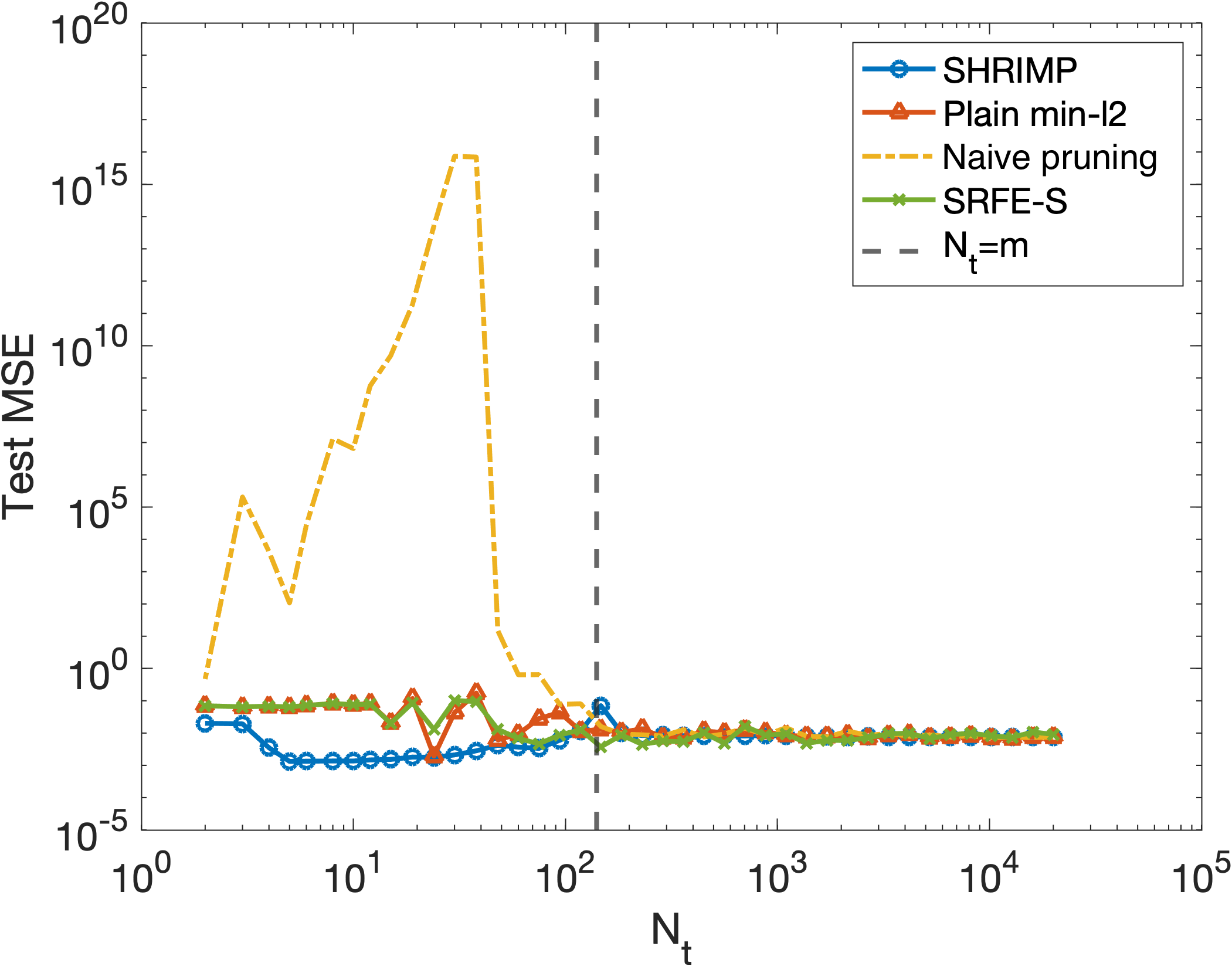}
    \caption{Test MSE of sparse random feature models obtained by SHRIMP, min $\ell_2$-norm estimation, min $\ell_1$-norm estimation (SRFE-S). From top to bottom and left to right: $f_2(\vx), f_5(\vx), f_7(\vx), f_3(\vx), f_4(\vx), f_6(\vx)$ as defined in Section \ref{sec:exp}, respectively.}
    \label{fig:naive}
\end{figure}
\subsection{Results on High-Order Functions}
In addition to low-order functions, we present test MSE curves of high-order functions with different methods in Figure \ref{fig:high_order} for completeness. All experiments are with $m=140, d=10, q=d$ since the ground-truth order $q^*=d$, and the same experimental settings as the synthetic experiments in Section \ref{sec:exp}. For $f_{h_1}(\vx)$ and $f_{h_2}(\vx)$, SHRIMP results in better performance with sparser models. However, for $f_{h_3}(\vx)$, which has underlying dense weights, SHRIMP can only have comparable performance to SRFE-S but with a sparser coefficient vector.  
\begin{figure}[ht]
    \centering
    \includegraphics[width=.32\linewidth]{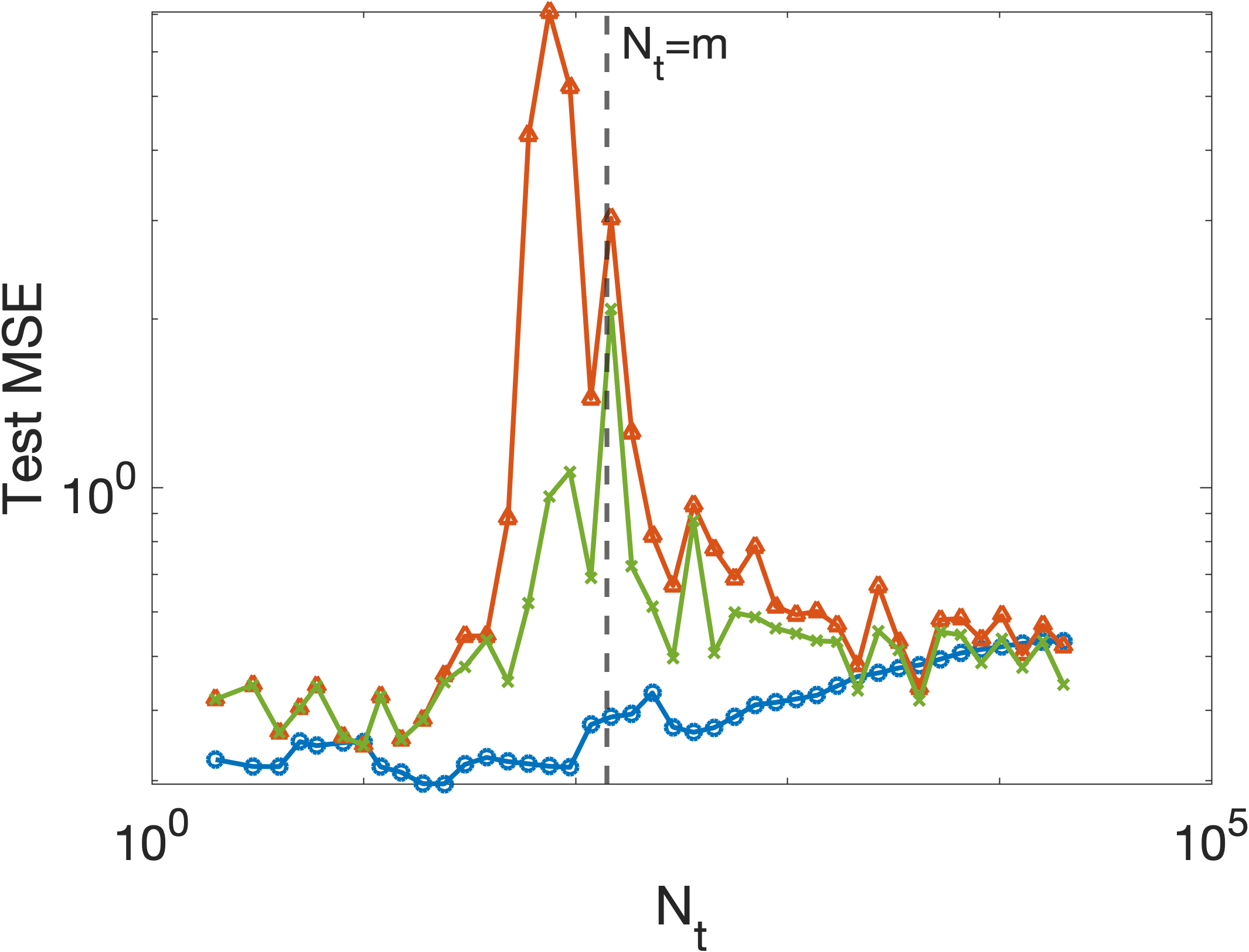}
    \includegraphics[width=.32\linewidth]{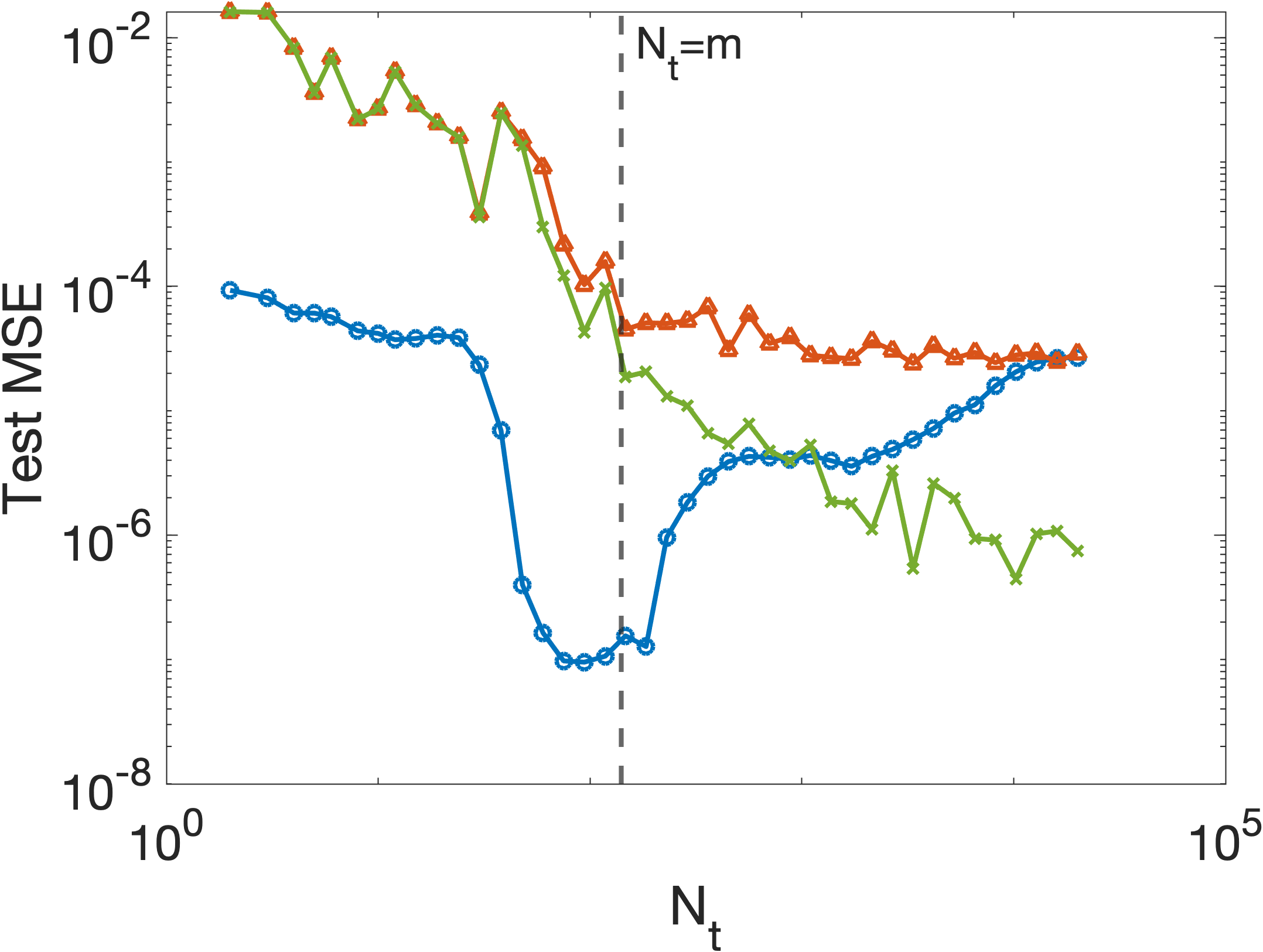}
     \includegraphics[width=.32\linewidth]{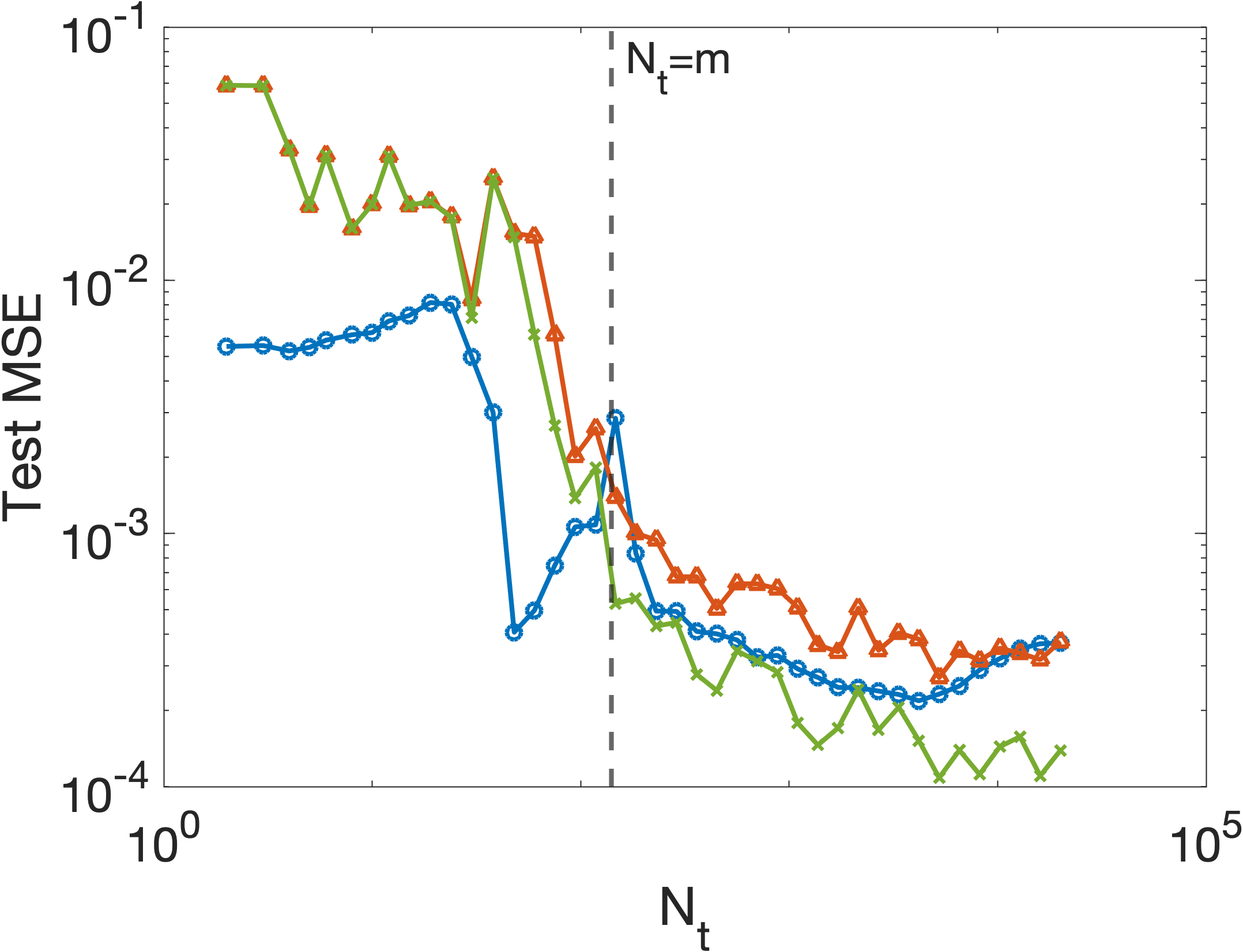}
    \caption{Test MSE of sparse random feature models obtained by SHRIMP, min $\ell_2$-norm estimation, min $\ell_1$-norm estimation (SRFE-S). Left: $f_{h_1}(\vx)=\sin(\sum_{i=1}^d x_i)$; Middle: $f_{h_2}(\vx)=\cos(\prod_{i=1}^d x_i)$; Right: $f_{h_3}(\vx)=(1+\|\vx\|_2)^{-1/2}$. Blue: SHRIMP; Orange: Plain min $\ell_2$; Green: SRFE-S.}
    \label{fig:high_order}
\end{figure}
\subsection{Experiments on Different Variances of Random Features}\label{app:variances}
We show the performance of random feature models with different variances of $\vw$ on functions with varying smoothness.  We compare SHRIMP and min $\ell_2$ solutions on the following functions:
\begin{itemize}
    \item Sum of low-frequency functions only (i.e., smooth function): $f_{a_1(\vx)} = \cos(x_1)x_3+x_2^2x_4+\sum_{j=3}^dx_j$
    \item Low-freq + High-freq:  $f_{a_2(\vx)} = \cos(x_1+x_2)+ 5\cos(2x_3+10x_4)$
    \item High-freq + High-freq: $f_{a_3(\vx)} = \sin(9x_1) + 10\cos (10 x_2)$
\end{itemize}
For all functions, we set $d=10, q=q_*, m=200, N=1500$, and average over 10 trials.  We give the test errors of SHRIMP vs the minimum $\ell_2$-norm solution at three different variances: $1/q, 1, 100$ in Table \ref{tab:variance}.  For the smooth function $f_{a_1}$ with only low-frequency component functions, SHRIMP outperforms the min $\ell_2$ norm at all corresponding variances, and as $\sigma^2$ increases, performance degrades.  This implies that sampling at lower frequencies may have an implicit bias toward smooth functions.  When there are higher-frequency component functions, such as $f_{a_2}$ and $f_{a_3}$, performance improves as $\sigma^2$ increases; in this case, we need to sample at higher frequencies.  Moreover, for $f_{a_3}$, increasing $\sigma^2$ allows for more gains: for this function a $\sigma^2\approx 170$ seems to do the best from a course sweep over $\sigma^2$ ranging up to 200.

We can visualize the behavior of different types of functions with high variance $\vw$ through the spectrum in Figure \ref{fig:variance1}. We plot the maximum eigenvalue of the Gram matrix throughout SHRIMP and random pruning for $f_{a_1}$ and $f_{a_3}$, where the weights are drawn from $\mathcal{N}(\mathbf{0}, 100\mI)$, a relatively high variance.  We notice that for $f_{a_1}$, the maximum eigenvalue of SHRIMP qualitatively matches that of random pruning throughout the pruning process.  However, for $f_{a_3}$, the behavior of the maximum eigenvalue of SHRIMP seems to more closely match that of SHRIMP with low variance on smooth functions (such as those given in the main paper), where the maximum eigenvalue of SHRIMP is smaller than that of random pruning for essentially the entire pruning process.  We discuss some explanations for this behavior in Section \ref{app:further}, as well as some limitations of our current theory in explaining this phenomenon.  

\begin{table}[ht]
    \centering
    \begin{tabular}{cccccccc}
    \toprule
    Model  & \multicolumn{3}{c}{min $\ell_2$} & \multicolumn{3}{c}{SHRIMP} & Optimal $q_*$ \\
    \hline
    $\sigma^2$ &  $1/q$ & 1 & 100 &  $1/q$ & 1 & 100 \\
    \hline
        $f_{a_1}$ &  0.041 & 0.034 & 2.452 & 1.19e-05 & 2.99e-04 & 0.137 & 2 \\
        $f_{a_2}$ &  55.205 & 40.023 & 9.734 & 14.411 & 12.778 & 4.935 & 2\\
        $f_{a_2}$ &  186.283 & 184.581 & 47.864 & 73.381 & 50.650 & 10.334 & 1\\
    \bottomrule
    \end{tabular}
    \caption{Test MSE of $\{f_{a_i}\}_{i=1}^3$ with random feature models with different variances.}
    \label{tab:variance}
\end{table}

\begin{figure}[ht]
    \centering
    \includegraphics[width=.49\linewidth]{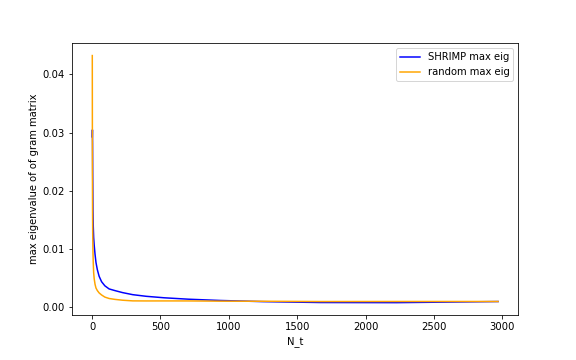}
     \includegraphics[width=.49\linewidth]{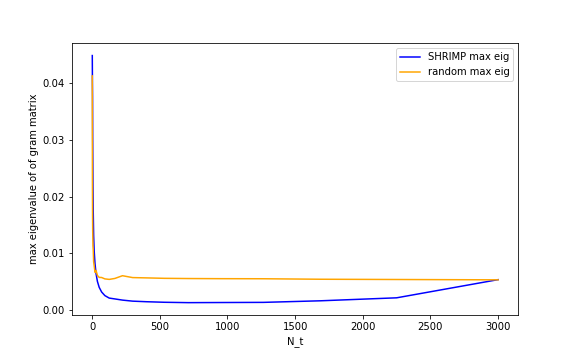}
    \caption{Maximum eigenvalue of $A_SA_S^\top/N_t$ with weight vectors drawn from $\mathcal{N}(\mathbf{0}, 100\mI)$.  Left: $f_{a_1}$; Right: $f_{a_3}$.}
    \label{fig:variance1}
\end{figure}

\subsection{Experiments on Kernel Approximation}
   We show an example of kernel approximation with $f(\vx)=x_4^2+x_2x_3+x_1x_2+x_4$ and $d=5$ in this section to illustrate the benefit of sparse random features beyond kernel approximation capacity.
    \begin{figure}[ht]
            \centering
            \includegraphics[width=.32\linewidth]{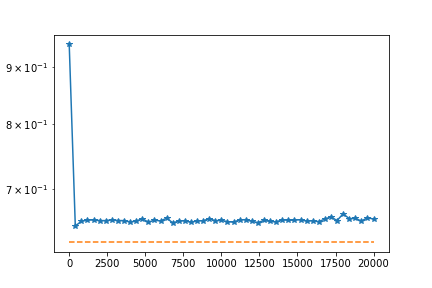}
            \includegraphics[width=.32\linewidth]{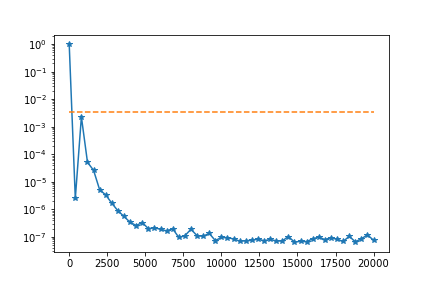}
            \includegraphics[width=.32\linewidth]{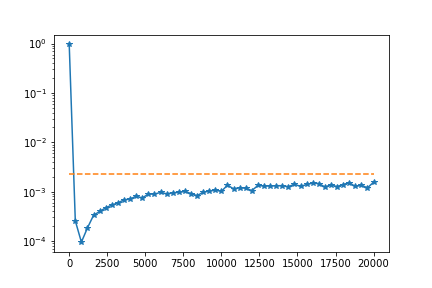}
            \includegraphics[width=.32\linewidth]{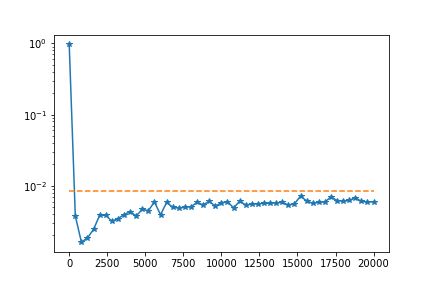}
            \includegraphics[width=.32\linewidth]{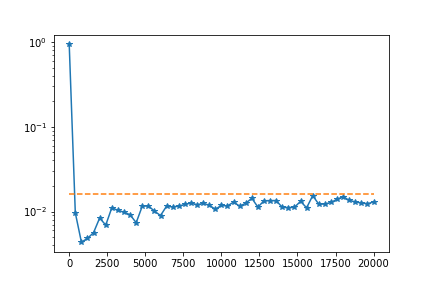}
            \caption{$f(\vx)=x_4^2+x_2x_3+x_1x_2+x_4$; From top to down and from left to right: $q=1, 2, 3, 4, 5$.}
            \label{fig:kernel_approximations}
        \end{figure}
    As Figure \ref{fig:kernel_approximations} shows, the minimal $\ell_2$ regression on a random features model is equivalent to minimum RKHS kernel regression on the kernel matrix corresponding to the random feature matrix. In Figure \ref{fig:kernel_approximations}, the blue lines represent the test MSE of the minimal $\ell_2$ estimator for sparse random feature models (i.e., $\hat{\vc} = \mA^{\dagger} \vy$) with increasing number of features ($N$), while the orange lines represent the test MSE of estimators from kernel regression with the kernel $\mK$ defined in Section \ref{sec:exp}.  Notice that the top-middle plot, $q=2$ equals the actual order $q_*$ of the function, and exhibits very interesting behavior as $N$ grows.  Instead of asymptotically approaching the orange line as the plots with $q=3, 4, 5$ do, the test error curve of the sparse random feature model with min $\ell_2$-norm estimator is significantly better than what can be obtained with kernel regression, which is surprising. This implies a \emph{benefit of sparse random feature models beyond a RKHS understanding}, which would necessitate a more robust statistical study beyond approximation capacities. 
\section{Proofs of Theorems in Section \ref{sec:theory}} \label{app:prop} 

\subsection{Proof of Theorem \ref{thm:refine}}
\begin{proof}[Proof of Theorem \ref{thm:refine}]
Denote $\vcsharp$ as the minimal $\ell_1$ norm solution obtained by basis pursuit, $\vcp$ as the pruned solution with zeros on $i \notin \ssharp$, where $\ssharp$ is the support set of the $s$ largest coefficients of $\vcsharp$, and $\vcstars$ as $\vcstar$ supported on $\sstar$, which is the support set of the $s$ largest coefficients of $\vcstar$. Since both $\vcp$ and  $\vcstars$ are $s$-sparse, we have for any $\vz$,
\begin{align}\label{eq:f_diff}
\begin{split}
    \abs{ \fsp (\vz) - \fstar(\vz)} ^2 
    & = \abs{ [ \phi(\vz; \vomega_1), \dots, \phi(\vz; \vomega_N) ] (\vcstar - \vcp )  }^2 \\
    & = \abs{ [ \phi(\vz; \vomega_1), \dots, \phi(\vz; \vomega_N) ] ( \vcp - \vcstars ) + (\vcstars - \vcstar)  }^2 \\
    & \leq 2 \abs{ [ \phi(\vz; \vomega_1), \dots, \phi(\vz; \vomega_N) ] ( \vcp - \vcstars )  }^2 +  2 \abs{ [ \phi(\vz; \vomega_1), \dots, \phi(\vz; \vomega_N) ] (\vcstars - \vcstar)  }^2\\
    & \leq 4s \| \vcstars - \vcp \|_2^2 + 2\abs{ \sum_{j\notin \sstar} \phi(\vz, \vomega_j) c_j^\star } ^2 \\
    & \leq 4s \| \vcstars - \vcp \|_2^2 + 2\abs{  \sum_{j\notin \sstar} \abs{ \phi(\vz, \vomega_j) } \abs{ c_j^\star}  }^2 \\
    & \leq 4s \| \vcstars - \vcp \|_2^2 + 2 (\kappa_{s, 1}(\vcstar) )^2
\end{split}
\end{align}

We provide two ways to bound $\|\vcstars - \vcp \|_2$ by using an alternative $\eta'$ or $\teta$ instead of $\eta$.

\begin{enumerate}
    \item \textbf{With the max singular value of $\mA$ (or $\lambda_{\max}(\mA^\ast\mA)$).} From (86) in \cite{hashemi2021generalization}, we have $\norm{\vy - \mA \vcstar}^2 \leq 2m (\epsilon^2 \norm{f}_\rho^2 + 4\nu^2)$ , then
\begin{align}
\begin{split}
     \|\vy - \mA \vcstars \|_2 & \leq  \|\vy - \mA \vcstar \|_2 +  \| \mA( \vcstar - \vcstars) \|_2 \\
     & \leq \sqrt{2m (\epsilon^2 \norm{f}_\rho^2 + 4\nu^2) } + \sqrt{\lambda_{\max} (\mA^\ast\mA)} \kappa_{s,2} (\vcstar) \\
    & = \eta' \sqrt{m} 
\end{split}
\end{align}

Then $\eta' = \sqrt{2(\epsilon^2 \norm{f}_\rho^2 + 4\nu^2)} + \sqrt{\frac{\lambda_{\max}(\mA^\ast\mA)}{m}} \kappa_{s,2} (\vcstar)$.

\item \textbf{With $\kappa_{s,1}$.} Follow the proof idea of (86), we have 

\begin{align}
    \begin{split}
        \|\vy - \mA \vcstars\|^2 & \leq 2 \p{\sum_{k=1}^m ( f(\vx_k) - \fstar(\vx_k)^2 + 4\nu^2 m } \\
        & \leq 2 \p{ \sum_{i=1}^m 2 ( f(\vx_k) - \fsstar(\vx_k) )^2 +  2 ( \fstar(\vx_k) - \fsstar(\vx_k) ) ^2 +  4\nu^2 m } \\
        & \leq 4m \p{ \epsilon^2\|f\|_{\rho}^2 + \kappa^2_{s,1}(\vcstar) + 2
        \nu^2} := \teta^2 m  
    \end{split}
\end{align}

where $( f(\vx_k) - \fsstar(\vx_k) ) ^2$ is bounded from Lemma 2 in \cite{hashemi2021generalization}. Then $\teta = 2\sqrt{\epsilon^2\|f\|_{\rho}^2  + 2
        \nu^2 + \kappa^2_{s,1}(\vcstar)} $.
\end{enumerate}

From Stability of BP-based Sparse Reconstruction \citep{foucartmathematical} (which is also Lemma 6 in \cite{hashemi2021generalization}), we have 
\begin{align}\label{eq:choice}
    \|\vcsharp - \vcstars \|_2 \leq C'\frac{\kappa_{s,1}(\vcstars)}{\sqrt{s}} + C \min\{ \eta', \teta 
    \} =  C \min\{ \eta', \teta \} 
\end{align}
where $\kappa_{s,1}(\vcstars) = 0$ since $\vcstars$ is $s$-sparse.
Then, if the coherence of $\mA$ satisfies $\mu_{\mA} \leq \frac{4}{\sqrt{41}(2s-1)}$, we have 
\begin{align}
    \begin{split}
        \|\vcp - \vcstars\|_2 \leq 3 \| \vcsharp - \vcstars \|_2 \leq 3 C \min\{ \eta', \teta \} := C \min\{ \eta', \teta \}
    \end{split}
\end{align}
by redefining $C$.

Then for McDiarmid's inequality, we have 
\begin{align}
\begin{split}
    \abs{\nu(\vz_k) - \nu(\tilde{\vz_k})} & \leq \frac{2 \p{  4s \| \vcstars - \vcp \|_2^2 + 2 (\kappa_{s, 1}(\vcstar) )^2 }  }{ m} := \Delta_{v} \\
   \Rightarrow t & = \Delta_{v}\sqrt{\frac{m}{2} \log \p{ \frac{1}{\delta} } } = 2\sqrt{\frac{2}{m} \log \p{ \frac{1}{\delta} } }  \p{  2s \| \vcstars - \vcp \|_2^2 + (\kappa_{s, 1}(\vcstar) )^2 }   
\end{split}
\end{align}

Putting everything together, we have 

\begin{align}
    \begin{split}
        \sqrt{ \int_{\sR^d} \abs{ \fsp (\vx) - \fstar(\vx) } ^2 d\mu }  \leq  m^{-\frac{1}{2}} \sqrt{ \sum_{k=1}^m \abs{ \fsp(\vz_k) - \fstar(\vz_k) } ^2 } + \p{ \frac{8}{m} \log \p{ \frac{1}{\delta} } } ^{\frac{1}{4}}  \p{  2s \| \vcstar - \vcp \|_2^2  + (\kappa_{s, 1}(\vcstar) )^2 } ^{\frac{1}{2}} 
    \end{split}
\end{align}

For the first term, following (96) in \cite{hashemi2021generalization}, we bound as follows:

    \begin{align}
    m^{-\frac{1}{2}} \sqrt{ \sum_{k=1}^m \abs{ \fsp(\vz_k) - \fstar(\vz_k) } ^2 } \leq 2 \norm{ \vcp - \vcstars }_2  + \kappa_{s,1}(\vcstar)
    \end{align}
    where $\mA$ satisfies $2s$-RIP condition and $\vcp - \vcstars$ is $2s$-sparse.

Therefore, with probability at least $1-\delta$, we have the refined bound as follows:
\begin{align}
\begin{split}
        \sqrt{ \int_{\sR^d} \abs{ \fsp (\vx) - \fstar(\vx) } ^2 d\mu }   \leq \p{ \frac{8}{m} \log \p{ \frac{1}{\delta} } } ^{\frac{1}{4}}  \p{  2s \| \vcstars - \vcp \|_2^2  + (\kappa_{s, 1}(\vcstar) )^2 } ^{\frac{1}{2}} 
     + 2 \norm{ \vcp - \vcstars }_2 + \kappa_{s, 1}(\vcstar)
\end{split}
\end{align}

where
$$ \| \vcstars - \vcp \|_2 \leq   C \min \left\{ \sqrt{2(\epsilon^2 \norm{f}_\rho^2 + 4\nu^2)} + \sqrt{\frac{\lambda_{\max}(\mA^\ast\mA)}{m}} \kappa_{s,2} (\vcstar), 2\sqrt{\epsilon^2\|f\|_{\rho}^2  + 2
        \nu^2 + \kappa^2_{s,1}(\vcstar)} 
    \right\}.$$

Furthermore, if we plug in $\kappa_{s,2}(\vcstar) \leq \epsilon \|f\|_{\rho}$ from (89) in \cite{hashemi2021generalization}, we have
\begin{footnotesize}
\begin{align}
    \begin{split}
       &  \sqrt{ \int_{\sR^d} \abs{ \fsp (\vx) - \fstar(\vx) } ^2 d\mu }   \leq \\ &  \p{ \frac{8}{m} \log \p{ \frac{1}{\delta} } } ^{\frac{1}{4}} 
        \left( 2s \left(  
       C \min \left\{ \sqrt{2(\epsilon^2 \norm{f}_\rho^2 + 4\nu^2)} + \sqrt{\frac{\lambda_{\max}(\mA^\ast\mA)}{m}} \epsilon \|f\|_{\rho},  2\sqrt{\epsilon^2\|f\|_{\rho}^2  + 2
 \nu^2 + \kappa^2_{s,1}(\vcstar) } \right\} \right)^2  +  \kappa^2_{s,1}(\vcstar) \right)^{\frac{1}{2}} \\
 & \quad + 2 C \min \left\{ \sqrt{2(\epsilon^2 \norm{f}_\rho^2 + 4\nu^2)} + \sqrt{\frac{\lambda_{\max}(\mA^\ast \mA)}{m}} \epsilon^2\|f\|_{\rho}, 2\sqrt{\epsilon^2\|f\|_{\rho}^2  + 2
        \nu^2 + \kappa^2_{s,1}(\vcstar)} 
    \right\} + \kappa_{s, 1}(\vcstar)
    \end{split}
\end{align}
\end{footnotesize}
    
\end{proof}

\subsection{Proof of Corollary \ref{cor:sparse}}
\begin{proof}[Proof of Corollary \ref{cor:sparse}]
For the Corollary, since $\vcstar$ is $s$-sparse, $\kappa_{s, p}=0, \forall p$ and $\vcstars - \vcstar = \mathbf{0}$. We can bound $\abs{ \fsp (\vz) - \fstar(
\vz) } ^2$ using almost the same way as \eqref{eq:f_diff}, but with a tighter constant. Since both $\vcp$ and $\vcstar$ are $s$-sparse, $ |\ssharp \bigcup \sstar| \leq 2s $ and $[\vcp -\vcstar]_i = 0, \forall i \notin  \ssharp \bigcup \sstar$. Then,
\begin{align}
\begin{split}
    \abs{ \fsp (\vz) - \fstar(\vz)} ^2 
    & = \abs{ [ \phi(\vz; \vomega_1), \dots, \phi(\vz; \vomega_N) ] (\vcstar - \vcp )  }^2 \\
    & = \abs{ [ \phi(\vz; \vomega_i)]_{i \in \ssharp \bigcup \sstar} ([\vcstar]_{i \in \ssharp \bigcup \sstar} - [\vcp]_{i \in \ssharp \bigcup \sstar})  }^2 \\
    & \leq 2s \| \vcstar - \vcp \|^2 
\end{split}
\end{align}

Hence, the bound in Lemma 7 (with $\eta = \sqrt{2(\epsilon^2\|f\|_{\rho} + E^2)}$ ) changes to 

\begin{align}
    \| \vcp - \vcstar \|_2 & \leq 3C \eta = 3 C \sqrt{2(\epsilon^2 \|f\|_{\rho}^2 + 4\nu^2)} : = C\sqrt{\epsilon^2 \|f\|_{\rho}^2 + 4\nu^2}
\end{align}
after redefining $C$.

Furthermore, we can bound the difference in $v$ from (91) in \cite{hashemi2021generalization} by 

\begin{align}
    \begin{split}
      \abs{\nu(\vz_k) - \nu(\tilde{\vz_k})} & \leq \frac{1}{m} \abs{ \abs{ \fsp(\vz_k) - \fstar(\vz_k) } ^2 -  \abs{ \fsp(\tilde{\vz}_k) - \fstar(\tilde{\vz}_k) } ^2  }    \leq \frac{4s}{m} \|\vcstar - \vcp \|_2^2 = \frac{4s C^2 (\epsilon^2 \|f\|_{\rho}^2 + 4\nu^2)   }{m} 
    \end{split}
\end{align}
where $\tilde{\vz_k}$ results from perturbing $\vz_k$ at the $k^{th}$ coordinate, and
\begin{align}
    t = \frac{4s C^2 (\epsilon^2 \|f\|_{\rho}^2 + 4\nu^2)   }{m} \sqrt{\frac{m}{2} \log \p{ \frac{1}{\delta} } } =  s C^2 (\epsilon^2 \|f\|_{\rho}^2 + 4\nu^2)   \sqrt{ \frac{8}{m} \log \p{ \frac{1}{\delta} } }.
\end{align}

Therefore, with probability exceeding $1-\delta$, 
\begin{align}
    \begin{split}
        \sqrt{ \int_{\sR^d} \abs{ \fsp (\vx) - \fstar(\vx) } ^2 d\mu } & \leq m^{-\frac{1}{2}} \sqrt{ \sum_{k=1}^m \abs{ \fsp(\vz_k) - \fstar(\vz_k) } ^2 } + \p{ \frac{8}{m} \log \p{ \frac{1}{\delta} } } ^{\frac{1}{4}}  s^{\frac{1}{2}} C \sqrt{\epsilon^2 \|f\|_{\rho}^2 + 4\nu^2} \\
        & \leq 2\| \vcp - \vcstar \|_2 + \p{ \frac{8}{m} \log \p{ \frac{1}{\delta} } } ^{\frac{1}{4}}  s^{\frac{1}{2}} C \sqrt{\epsilon^2 \|f\|_{\rho}^2 + 4\nu^2} \\
        & \leq C \p{ 2 + 8^{\frac{1}{4}}s^{\frac{1}{2}}m^{-\frac{1}{4}} \log^{\frac{1}{4}}(1/
        \delta) } \sqrt{\epsilon^2 \|f\|_{\rho}^2 + 4\nu^2 } 
    \end{split}
\end{align}

With order-$q$ features, we have $\eta = \sqrt{ 2\epsilon^2 \binom{d}{q} \tvert{f}^2 + 2E^2 }$, where $\tvert{f} = \frac{1}{K}\sum_{j=1}^K\|g_j\|_{\rho}$. Hence, the bound reduces to 

\begin{align}
    \begin{split}
        \sqrt{ \int_{\sR^d} \abs{ \fsp (\vx) - \fstar(\vx) } ^2 d\mu } & \leq C \p{ 2 + 8^{\frac{1}{4}}s^{\frac{1}{2}}m^{-\frac{1}{4}} \log^{\frac{1}{4}}(1/
        \delta) } \sqrt{\epsilon^2 \binom{d}{q} \tvert{f}^2 + E^2 } \\
        & := \gO \p{ \p{ 1 + C's^{\frac{1}{2}}m^{-\frac{1}{4}} \log^{\frac{1}{4}}(\frac{1}{\delta}) } \sqrt{\epsilon^2 \binom{d}{q} \tvert{f}^2 + E^2 } }
    \end{split}
\end{align}

\end{proof}

\subsection{Proof of Proposition \ref{prop:spectrum}}
\begin{proof}[Proof of Proposition \ref{prop:spectrum}]
Let $X_\ell \in \mathbb{C}^m$ be the $\ell$th row of $\mA^\ast$ for $\ell \in [N]$, i.e., 
\begin{align*}
X_\ell &= [  \overline{\phi(x_1,\vomega_\ell)}, \ldots, \overline{\phi(x_m,\vomega_\ell)}]\\
&= [  \overline{\phi(x_1,\vomega_\ell)}, \ldots, \overline{\phi(x_m,\vomega_\ell)}].
\end{align*}
We can decompose $\frac{1}{N} \mA \mA^\ast$ into the following sum of rank-1 matrices:
\begin{align}
\frac{1}{N}\mA  \mA^\ast = \frac{1}{N}\sum_{\ell=1}^N X_\ell^\ast X_\ell = \frac{1}{N}\sum_{\ell=1}^N [  {\phi(x_1,\vomega_\ell)}, \ldots, {\phi(x_m,\vomega_\ell)}]^T [  \overline{\phi(x_1,\vomega_\ell)}, \ldots, \overline{\phi(x_m,\vomega_\ell)}].
\end{align}
For a fixed $\ell$, each component of $X_\ell^\ast X_\ell$ takes the form $ \exp(i \langle \vx_j-\vx_k ,\vomega_\ell \rangle)$.

Here, we quantify the effect of the conditioning of the linear system $\mA \mA^\star$ (overparameterized setting) as  $N\rightarrow m^+$. Since $\min(m,N)=m$, we will bound $\lambda_{m}\left( \frac{1}{N} \mA  \mA^\ast \right)$ and $\lambda_{1}\left( \frac{1}{N} \mA  \mA^\ast \right)$. 

The Rayleigh quotient can be bounded by
\begin{equation}
\lambda_{m}\left(\frac{1}{N} \mA  \mA ^\ast \right) \leq \frac{1}{N} \langle \mA^\ast \vv,   \mA^\ast \vv  \rangle.
\end{equation}
for all unit vectors $\vv \in \mathbb{C}^m$. The set $\gX =\{X_1, \ldots, X_{m-1}\}$ forms a subspace of $\mathbb{C}^m$ of dimension at most $m-1$, thus there exists a unit vector $\vz \in \mathbb{C}^m$ orthogonal to $span(\gX)$. And
\begin{align}
\begin{split}
    \lambda_{m} \left( \frac{1}{N} \mA  \mA^\ast \right) &\leq \frac{1}{N}\langle  \vz, \mA\mA^\ast \vz  \rangle \\
    &\leq \frac{1}{N} \sum_{\ell=m}^N \vz^\ast X_\ell^\ast X_\ell \vz  \\
    & = \frac{1}{N} \sum_{\ell=m}^N \sum_{j,k=1}^m \overline{\vz_j} \, \vz_k \exp(i\langle \vx_j-\vx_k, \vomega_\ell \rangle)\\
     & \leq \frac{N-m+1}{N}+\frac{1}{N} \sum_{\ell=m}^N \sum_{j\neq k} \overline{\vz_j} \, \vz_k \exp(i\langle \vx_j-\vx_k, \vomega_\ell \rangle).
\end{split}
\end{align}
The vector $\vz$ is independent of $\vomega_\ell$ for $\ell\geq m$, therefore, applying expectations lead to:
\begin{align}
\begin{split}
    \mathbb{E} \lambda_{m}\left( \frac{1}{N}\mA \mA^\ast \right) &\leq 
    \frac{N-m+1}{N} + \frac{1}{N}\,  \mathbb{E} \sum_{\ell=m}^N\sum_{j\neq k} \overline{\vz_j} \, \vz_k \exp(i\langle \vx_j-\vx_k, \vomega_\ell \rangle)   \\ 
   &=    \frac{N-m+1}{N}  +  \frac{1}{N} \, \mathbb{E}_{\vx} \, \mathbb{E}_{\vomega_1,\ldots, \vomega_{m-1}} \sum_{\ell=m}^N \sum_{j\neq k} \overline{\vz_j} \vz_k \mathbb{E}_{\vomega_\ell}[\exp(i\langle \vx_j-\vx_k, \vomega_\ell \rangle)]\\
       &=\frac{N-m+1}{N} +  \ \frac{1}{N} \, \mathbb{E}_{\vx} \, \mathbb{E}_{\vomega_1,\ldots, \vomega_{m-1}} \sum_{\ell=m}^N \sum_{j\neq k} \overline{\vz_j}\vz_k \exp\left(-\frac{\sigma^2}{2}\|\vx_k-\vx_j\|_2^2\right)  \\
    &\leq \frac{N-m+1}{N} + \frac{1}{N} \sum_{\ell=m}^N \mathbb{E}_{\vx}\sqrt{\sum_{j\neq k} \exp\left(-\sigma^2\|\vx_k-\vx_j\|_2^2\right) } \\
    &\leq \frac{N-m+1}{N}+ \frac{(N-m+1)\sqrt{m^2-m}}{N}\left(4\gamma^2\sigma^2+1\right)^{-\frac{q}{4}}
\end{split}
\end{align}

using Holder's and Jensen's inequalities (noting $\|\vz\|^4_2=1$).  Repeating for the maximum eigenvalue:
$$
\lambda_1\left( \frac{1}{N} \mA  \mA^\ast \right) \geq \frac{1}{N} \langle \vz, \mA \mA^\ast \vz \rangle
$$
and setting the unit vector to $z = \frac{1}{\sqrt{N}} {X}_1$ yields
\begin{align}
\begin{split}
    \lambda_1\left( \frac{1}{N} A A^\ast \right) &\geq \frac{1}{N^2} \sum_{\ell=1}^N X_1^\ast  X_\ell X_\ell^\ast  X_1   \\
    &=\frac{1}{N^2}\left( N^2 + \sum_{
    \ell=2}^N {X}_1^\ast  {X}_\ell{X}_\ell^\ast  {X}_1 \right)  \\
    & = 1 + \frac{1}{N^2} \sum_{\ell=2}^N \sum_{j,k=1}^m \exp(i\langle \vx_k - \vx_j, \vomega_1-\vomega_\ell \rangle)  \\
     & = 1 + \frac{(N-1)m}{N^2}+ \frac{1}{N^2}\sum_{\ell=2}^N\, \sum_{\substack{j,k=1\\ j\neq k}}^m \exp(i\langle \vx_k - \vx_j, \vomega_1-\vomega_\ell \rangle),
\end{split}
\end{align}
and thus
\begin{align}
    \mathbb{E}\lambda_1\left( \frac{1}{N}A A^\ast \right)\geq 2 - \frac{(N-1)m}{N^2}+ \frac{(N-1)(m^2-m)}{N^2}\left(4\gamma^2\sigma^2+1\right)^{-\frac{q}{4}}.
\end{align}
Consider a linear scaling $N = c m$ for $c>1$, then 
\begin{align}
\begin{split}
     \mathbb{E} \lambda_{m} \left(\frac{1}{N}A A^\ast\right) &\leq 
   \frac{(c-1)m+1}{cm}+ \frac{((c-1)m+1)\sqrt{m^2-m}}{cm}\left(4\gamma^2\sigma^2+1\right)^{-\frac{q}{4}}\\
   &\leq 
   \frac{c-1}{c}+ \frac{1}{m} +\frac{c-1}{c} m \left(4\gamma^2\sigma^2+1\right)^{-\frac{q}{4}} 
  + \left(4\gamma^2\sigma^2+1\right)^{-\frac{q}{4}}.
\end{split}
 \end{align}
\end{proof}
\section{Further Discussion}\label{app:further}

In this section, we provide further discussion and future research directions on how SHRIMP connects to previous work on $\ell_0$-regularized problems and random features approximations.

\subsection{Connection to $\ell_0$-regularized Problems}

We first note the similarities of SHRIMP with the SINDy \citep{zhang2019convergence} algorithm. SINDy prunes all features whose magnitude lies below a fixed $\lambda$. Thus, SHRIMP can be considered as an adaptive form of SINDy, where the $\lambda$ is chosen adaptively at each step depending on the data.  \citet{zhang2019convergence} prove under certain conditions that SINDy converges to a fixed point, which is a local minimizer of the non-convex $\ell_0$-regularized regression problem with regularization parameter $\lambda$.  However, the proof requires the matrix $\mA$ to be full column-rank, which in our case, is not guaranteed; in fact, one of the advantages of SHRIMP is that it can often move from the overparameterized to the underparameterized setting.

A more general perspective on the minimizers of the $\ell_0$-regularized regression problem is given in \cite{nikolova2013description}. Remark 2 in \cite{nikolova2013description} indicates that SHRIMP solves an $\ell_0$-regularized minimization problem at each step.  Additionally, Theorem 3.2 indicates that when $\mA$ has full rank with probability one and when SHRIMP prunes in the underparameterized setting, each iterate is a \textit{strict} local minimum of the regularized regression problem.  Thus, SHRIMP can be seen as solving for local minimizers in a data-dependent sequence of $\ell_0$-regularized regression problems.

However, there are still a few gaps in this direction.  First, we focus on generalization performance of pruned models, while most work on $\ell_0$-regularization focus on sparse recovery.  Moreover, the results in \citet{nikolova2013description} indicate that solutions of problems that SHRIMP is solving are local minimizers of the regularized regression problem for \textit{any} regularization parameter greater than zero, but it is hard to compare which local minimizer is the best regarding test performance.

\subsection{The variance of the random feature weights}

In our experiments in Section \ref{sec:exp}, we set the variance of the $\mathbf{w}$ to be $1/q$, or the inverse of the order of the function assuming oracle access.  This is motivated by two reasons: first, in deep learning initialization \citep{he2015delving, kumar2017weight} and corresponding theory \citep{ba2019generalization}, the variance of the weights in a layer is the inverse of the input dimension or the number of weights in the layer.    Second, consider the bound for kernel approximation given in \citet{rahimi_random_2008}, which is
\begin{equation}\label{eq:kernel_approx}
    \mathrm{Pr}\left[\sup_{\mathbf{x}, \mathbf{y}\sim \mathcal{M}}|z(\mathbf{x})'z(\mathbf{y})-k(\mathbf{x}, \mathbf{y})|\geq \epsilon\right]\leq 2^8\left(\frac{\sigma\,\mathrm{diam}(\mathcal{M})}{\epsilon}\right)^2\mathrm{exp}\left(-\frac{N\epsilon^2}{4(d+2)}\right),
\end{equation}
where $\mathcal{M}$ is a compact set. 
 
If the data comes from a Gaussian distribution, with high probability they lie in a compact set ($\mathcal{M}$).  When the dimension ($d$) increases, the diameter $ \mathrm{diam}(\mathcal{M})$ grows as $\gO(\sqrt{d})$. If we use low-order features to approximate, the effective dimension is $q$, so the diameter grows as $\gO(\sqrt{q})$.  Thus, setting $\sigma=1/\sqrt{q}$ mitigates this increase by allowing the numerator to be $\gO(1)$.    

However, the theory with respect to basis pursuit requires a variance that \textit{increases} with $q$. Our experiments in Section \ref{app:variances} corroborate this theoretical gap.  While smaller variance may suffice to learn smooth or low-frequency functions, setting the variance to be small may not allow enough high-frequency weights to be sampled to learn higher frequency functions.  One possible explanation is that these higher frequency functions are better represented by functions in the RKHS corresponding to kernel parameter matching that of the larger variances in the random features.  However, as shown in \eqref{eq:kernel_approx}, the kernel approximation is much worse when $\sigma$ is large.  Thus, understanding this behavior is an interesting future direction.  

\end{document}